\documentclass{article} 

\usepackage{amsmath,amssymb}
\usepackage{bookmark}
\usepackage{graphicx,url}
\usepackage{amsthm}
\usepackage{authblk}
\usepackage{hyperref}
\usepackage{cleveref}
\usepackage[margin=1in]{geometry}

\newtheorem{theorem}{Theorem}
\newtheorem{lemma}[theorem]{Lemma}
\newtheorem{proposition}[theorem]{Proposition}
\newtheorem{remark}[theorem]{Remark}
\newtheorem{corollary}[theorem]{Corollary}
\newtheorem{definition}[theorem]{Definition}

\title{Generative Modeling with Denoising Auto-Encoders and Langevin Sampling}
\date{}
\author[1]{Adam Block}
\author[2]{Youssef Mroueh}
\author[1]{Alexander Rakhlin}
\affil[1]{MIT}
\affil[2]{IBM Research \& MIT-IBM Watson AI Lab}

\hypersetup{
	colorlinks=true,
	linkcolor=blue,
	filecolor=blue,      
	urlcolor=blue,
	citecolor=blue
}

\usepackage{times}

\DeclareMathOperator{\osc}{osc}
\DeclareMathOperator{\divv}{div}
\DeclareMathOperator*{\argmin}{argmin}
\newtheorem{assumption}{Assumption}
\begin{document}

\maketitle

\begin{abstract}
    We study convergence of a generative modeling method that first estimates the score function of the distribution using Denoising Auto-Encoders (DAE) or Denoising Score Matching (DSM) and then employs Langevin diffusion for sampling. We show that both DAE and DSM provide estimates of the score of the Gaussian smoothed population density, allowing us to apply the machinery of Empirical Processes. 
    We overcome the challenge of relying only on $L^2$ bounds on the score estimation error and provide finite-sample bounds in the Wasserstein distance between the law of the population distribution and the law of this sampling scheme. We then apply our results to the homotopy method of \cite{Ermon} and provide theoretical justification for its empirical success.
\end{abstract}

\section{Introduction}
     Recent empirical successes of generative modeling range from high-fidelity image generation  with Generative Adversarial Networks (GANs) \cite{GANoriginal} to protein folding with differentiable simulators \cite{senior2020improved,ingraham2018learning}.  GANs are implicit likelihood models, in the sense that they do not directly model the likelihood. On the other hand, explicit generative models directly estimate the likelihood or the score function (the gradient of the $\log$ likelihood). Recent works, including \cite{Ermon,Nguyen_2017,grathwohl2019your}, show that successful image generation can be achieved by estimating the score function from data (using Denoising Score Matching, Denoising Auto-Encoders, and pre-trained classifiers respectively) and by using Langevin dynamics for sampling. Conditional text generation, as well as protein folding can be accomplished using a similar approach, see e.g \cite{dathathri2019plug} for text generation and \cite{ingraham2018learning} for learning 3D protein structures from sequences. 
    Motivated by the recent resurgence of those explicit methods and their empirical success in wide range of applications we focus on this family of explicit generative models and address their theoretical properties.

    Formally, we consider the problem of data generation from an unknown distribution.  The algorithm we consider is in two parts.  In the first part, we estimate the score of the data using a Denoising Auto-Encoder (DAE), while in the second part we plug our estimate of the score of the data into a discretized approximation to the Langevin dynamics stochastic differential equation, generating a sample.  In this paper we bound the Wasserstein 2-distance between the law of our sampling algorithm and the population law. The algorithm is a variation on the method of \cite{Ermon}, which produced state-of-the-art results on standard vision datasets.
    \par
    If we consider a distribution with density $p$ with respect to the Lebesgue measure, then Langevin Dynamics provide a well-known and much studied way to sample from this distribution.  Supposing that $\nabla \log p$ is $\frac M2$-Lipschitz, we note that under mild conditions, the Langevin diffusion given by
    \begin{align}\label{eq:langevin}
        d W(t) = \nabla \log p(W(t)) d t + \sqrt{2} d B_t && W(0) \sim \mu_0
    \end{align}
    where $B_t$ is a standard Brownian motion in $\mathbb{R}^d$, converges in law as $t \to \infty$ to the population distribution $p$.  We consider the setting where $p$ is unknown, but we have access to $n$ i.i.d. samples $X_1, \dots, X_n \sim p$.  In this case, $\nabla \log p$, often called the score of the distribution, must be estimated.  We build on the observations of \cite{Vincent,Alain} and show that DAEs trained on the sample provide estimators of the score of a smoothed distribution that are close in the sense of $L^2$.  We then show that this estimate suffices to bring the Langevin process associated to our estimate close to the Langevin process that is actually associated with the population distribution.  Finally we note that our error decomposition lends itself naturally to the homotopy method used in \cite{Ermon} and we show that this approach significantly helps the sampling scheme with respect to Wasserstein-2 distance.  The new contributions are as follows:
    \begin{itemize}

        \item We show that an estimator of score that is close in the $L^2$ sense still furnishes us with a Langevin diffusion whose law at a fixed time $t > 0$ remains close in the sense of Wasserstein to the law of the Langevin diffusion driven by the score of the population distribution.  In particular, many statistical estimators are only guaranteed to have small $L^2$ error, rather than small uniform error, and so the ability to provide estimates of Wasserstein distance for a Langevin sample using such estimates has the potential for significant general application. While the theory is more straightforward for an estimator that is uniformly close, a case that is studied in \cite{RakhlinRaginsky}, to our knowledge this is the first theoretical justification for the $L^2$-close regime.
        
        \item We exhibit a decomposition of the error between the law of our sampling algorithm and the law of the population distribution that makes the benefits of a homotopy method explicit, thereby providing the first rigorous, theoretical justification that we know of for the success realized by the algorithm in \cite{Ermon}.
        
        \item We shed new light on what the Denoising Auto-Encoder (DAE) learns from a distribution, showing that optimizing DAE loss is equivalent to optimizing Denoising score-matching (DSM) loss, which then implies that the population level DAE provides an unbiased estimator of the score of the population distribution convolved with a Gaussian.  Moreover we provide a condition to guarantee that the convolved distribution satisfies a log-Sobolev inequality, the first such result of which we know.
        
        \item We use our connection between DAE and DSM losses to provide finite sample, high probability estimates of the error of a DAE.  To our knowledge, these are the first such finite sample bounds.
    \end{itemize}
     \par We consider a sampling scheme where we fix in advance a sequence $(\eta_i, \sigma_i^2)$ with both $\eta_i$ and $\sigma_i^2$ non-increasing, $1 \leq i \leq N$.  We then consider a sequence of DAEs $\widetilde{f}_1, \dots, \widetilde{f}_N$ trained on the data with the variance parameter of $\widetilde{f}_i$ equal to $\sigma_i^2$.  Let $\widehat{f}_i(x) = \frac 1{\sigma_i^2} (\widetilde{f}_i(x) - x)$.  Then we apply a homotopy method of discrete Langevin sampling with warm restarts where on the $i^{th}$ leg of the homotopy, we use $\widehat{f}_i$ as an estimate of score and a step length of $\eta_i$.  This is the identical algorithm proposed and empirically justified by \cite{Ermon}, up to the fact that we consider the DAE criterion, while they consider the DSM criterion.
\section{Notation and Preliminaries}
    There are two steps to our sampling scheme: the estimation of the score and the Langevin sampling.  For the first, we have
    \begin{definition}
        If $p$ is a density on $\mathbb{R}^d$, we call $\nabla \log p$ the score of the distribution.  We let $g_{\sigma^2}$ denote the density of a centred Gaussian with variance $\sigma^2$ and let $p_{\sigma^2} = p \ast g_{\sigma^2}$ denote the convolved distribution.  If $r: \mathbb{R}^d \to \mathbb{R}^d$ is a function, we denote the DAE error by
        \begin{equation}
            L_{DAE}(r) = \mathbb{E}_{\substack{X \sim p \\ \epsilon \sim g_{\sigma^2}}} \left[ ||r(X + \epsilon) - X||^2 \right]
        \end{equation}
        We add a hat to indicate that we are considering the empirical distribution:
        \begin{equation}
            \widehat{L}(r) = \frac {1}{n} \sum_{i = 1}^n ||r(X_i + \epsilon_i) - X_i||^2
        \end{equation}
        where $\epsilon_i$ are i.i.d. centred Gaussians of variance $\sigma^2$.  
        We define the Denoising Score-Matching (DSM) loss as
        \begin{equation}
            L_{DSM}(s) = \mathbb{E}_{X \sim p_{\sigma^2}}[||s(X) - \nabla \log p_{\sigma^2}(X)||^2]
        \end{equation}
    \end{definition}
    For the entirety of the paper, we assume that $\nabla \log p$ is $\frac M2$-Lipschitz.  Regarding the Langevin process, we have
    \begin{definition}
        We define the Langevin diffusion started at some distribution $\mu_0$ as the (guaranteed unique by say \cite{Karatzas}) solution $W_{\sigma^2}(t)$ to
         \begin{align}
            d W_{\sigma^2}(t) = \nabla \log p_{\sigma^2}(W_{\sigma^2}(t)) d t + \sqrt{2} d B_t && W_{\sigma^2}(0) \sim \mu_0
        \end{align}
        where we drop the $\sigma^2$ when context allows. 
        If $\widehat{f}$ is an $\frac{M}{2}$-Lipschitz estimate of $\nabla \log p_{\sigma^2}$, then we let $\widehat{W}(t)$ be the unique solution to
        \begin{align}
            d \widehat{W}(t) = \widehat{f}\left(\widehat{W}(t)\right) d t + \sqrt{2} d B_t && \widehat{W}(0) \sim \mu_0
        \end{align}
        For a fixed small step length $\eta > 0$, we define
        \begin{equation}
            W_{k+1} = \eta \widehat{f}\left(W_k\right) + \sqrt{2 \eta} \xi_k
        \end{equation}
        where $\xi_k$ are i.i.d. standard Gaussians in $\mathbb{R}^d$.
    \end{definition}
    We make use of the following definition
    \begin{definition}
        For constants $m, b > 0$, we say that a vector field $f: \mathbb{R}^d \to \mathbb{R}^d$ is $(m,b)$-dissapitive if for all $x \in \mathbb{R}^d$, we have
        \begin{equation}
            \langle f(x), x \rangle \geq  m ||x||^2 - b
        \end{equation}
    \end{definition}
    We also introduce some of the notation related to the theory of empirical processes.  
    \begin{definition}
        Let $\mathcal{G}$ be a class of real valued functions on $\mathbb{R}^d$ and let $S = (X_1, \dots, X_n)$ be $n$ samples from $\mathbb{R}^d$.  We define the Rademacher average with respect to the sample as
        \begin{equation}
            \widehat{\mathfrak{R}}_n(\mathcal{G}, S) = \mathbb{E}_\varepsilon \left[\sup_{g \in \mathcal{G}} \frac 1n \sum_{i =1}^n \varepsilon_i g(X_i) \right]
        \end{equation}
        where $\varepsilon_i$ are i.i.d random variables independent of the $X_i$, taking values $\{\pm1\}$ with probability $\frac 12$ each and the expectation is conditional on the $X_i$.  We define the Rademacher complexity of the function class $\mathcal{G}$ as
        \begin{equation}
            \mathfrak{R}_n(\mathcal{G}) = \sup_{S \subset \left(\mathbb{R}^d\right)^n} \widehat{\mathfrak{R}}_n(\mathcal{G}, S).
        \end{equation}
        For a class of $\mathbb{R}^k$-valued functions $\mathcal{G}$, we denote by  $\mathfrak{R}_n(\mathcal{G})= \sum_{i=1}^k \mathfrak{R}_n(\mathcal{G}_i)$
        where $\mathcal{G}_i$ is the restriction of $\mathcal{G}$ to the $i$-th coordinate.
    \end{definition}
    We make use of the following assumptions on the population distribution:
    \begin{assumption}
        The density $p$ is positive everywhere on $\mathbb{R}^d$.
    \end{assumption}
     \begin{assumption}\label{cor1}
 	The vector field $\nabla \log p$ is $\frac M2$ Lipschitz for some $M > 0$ and there exist $\widetilde{\sigma}_{max}^2, M > 0$ such that for all $0 \leq \sigma^2 \leq \widetilde{\sigma}_{max}^2$, the vector field $\nabla \log \frac {p_{\sigma^2}}{p}$ is $\sigma^2 \frac M2$-Lipschitz and similarly for $\widehat{f}$.
	 \end{assumption}
    \begin{assumption}
        The vector fields $-\nabla \log p$ and $\widehat{f}$ are $(m,b)$-dissipative for some positive constants $m,b$.
    \end{assumption}
    The first two assumptions are standard and are used to ensure that there exists a unique strong solution to the Langevin diffusion.  The last assumption has seen increased use in recent work to bound the log-Sobolev constant of $p$ when $-\log p$ is not strongly convex, as in, for example, \cite{RakhlinRaginsky}.  Note that the third condition coupled with the fundamental theorem of calculus implies that $p$ is $\frac 2m$-sub-Gaussian, as per \Cref{subgaussian}.
    
    Regarding the initialization of the Langevin algorithm, i.e. the distribution of $W(0)$ in \Cref{eq:langevin}, we make the assumption
    \begin{assumption}\label{assumption4}
		The initialization of $W_0 = W(0)$ satisfies the condition that for some $\alpha \geq 2 M^2$ with $M$ the Lipschitz constant above that there is some $k_\alpha$ such that $\mathbb{E}\left[e^{\alpha ||W(0)||^2}\right] = k_\alpha < \infty$ as well as $KL(\mu_0 , p_{\sigma^2}) < \infty$.
    \end{assumption}
	Note that in practice $W(0)$ is often initialized by a Gaussian with sufficiently small variance, certainly satisfies the above assumption.  This technical presupposition is required in order to control Wasserstein distance by relative entropy as will be seen below.  Below, we always take $\alpha < m$ when doing computations; the $\alpha = 2 M^2$ case is only to allow us to apply Girsanov's theorem and thus the fact that $k_\alpha < \infty$ is sufficient.
\section{Estimating the Score}\label{sec:estim}
    Our first result connects the DAE to Denoising Score-Matching (DSM), showing that the objectives are equivalent up to an affine transformation, a variant on a result in \cite{Vincent}; while \cite{Vincent} proves the following for a particular DAE paramaterization, we show that the below is true in general:
    \begin{proposition}\label{prop1}
        Let $p$ be a differentiable density.  Then the DAE loss
        \begin{equation}
            L_{DAE}(r) = \mathbb{E}_{x \sim p} \mathbb{E}_{\epsilon \sim g_{\sigma^2}} \left[ || r(x + \epsilon) - x ||^2 \right]
        \end{equation}
        and the DSM loss
        \begin{equation}
            L_{DSM}(s) = \mathbb{E}_{p_{\sigma^2}} \left[||s(x) - \nabla \log p_{\sigma^2}(x)||^2\right]
        \end{equation}
        with
        \begin{equation}
            s(x) = \frac{r(x) - x}{\sigma^2}
        \end{equation}
        are equivalent up to a term that does not depend on $r$ or $s$.
    \end{proposition}
    The proof of \Cref{prop1} is a simple application of the divergence theorem and Stein's lemma (\Cref{steinlemma}) and can be found in \Cref{app1}. \par
    An easy corollary of \Cref{prop1} is similar to a result in \cite{Alain}, which relates the population DAE to the score of the population distribution.  Instead, we consider the score of the population distribution smoothed by the addition of Gaussian noise; obviously, without knowledge of $p$, the DSM loss cannot be explicitly evaluated, so this equivalence allows for a loss that can be evaluated in practice.  Below, we establish \Cref{firstcor} using \Cref{prop1}; an alternate, direct proof, is included in \Cref{app1} for those interested.
    \begin{corollary}\label{firstcor}
		Let $p$ be a population density with respect to the Lebesgue measure.  Let $r_{\sigma^2}(x)$ denote the optimal DAE with Gaussian noise of variance $\sigma^2$.  Then
		\begin{equation}
			r_{\sigma^2}(x) = x + \sigma^2 \nabla \log p_{\sigma^2}(x)
		\end{equation}
	\end{corollary}
	\begin{proof}
	    Clearly $s(x) = \nabla \log p_{\sigma^2}(x)$ minimizes the $L_{DSM}$ loss.  Note that, by \Cref{prop1}, we have that $r(x) = x + \sigma^2 s(x)$ then minimizes $L_{DAE}$, the DAE loss.  The result follows.
	\end{proof}
	Later, in our analysis of the Langevin dynamics, we will need to assume dissipative conditions not just on the score of the population distribution $p$, but also on the score of the smoothed distribution $p_{\sigma^2}$; in particular, we will need to establish a log-Sobolev inequality for the smoothed distribution.  For reasons to become clear below, we define
	\begin{equation}
	    \sigma_{max}^2 = \frac{m}{2 M} \wedge \widetilde{\sigma}_{\max}^2
	\end{equation}
	where $\widetilde{\sigma}_{max}^2$ is as appears in \Cref{cor1}.  Then we have
	 \begin{proposition}\label{cor2}
	    Suppose that $-\nabla \log p$ is $\frac M2$-Lipschitz and $(m,b)$-dissipative.  Then for $\sigma^2 \leq \sigma_{\max}^2$, there is a constant $B$ such that $-\nabla \log p_{\sigma^2}$ is $(m_{\sigma^2}, b_{\sigma^2})$-dissipative where
	    \begin{align}
	        m_{\sigma^2} = \frac{m - \sigma^2 M}{2} && b_{\sigma^2} = b + \frac{B^2}{2(m - \sigma^2 M)}
	    \end{align}
	   \end{proposition}
	   \begin{remark}
	        Note that as $\sigma_{max}^2 \leq \frac{m}{2M}$, we can bound the dissapitivity constants to say that $-\nabla \log p_{\sigma^2}$ is $\left(\frac m4, b + \frac{m B^2}{4 M^2} \right)$-dissipative uniformly in $\sigma^2 \leq \sigma_{max}^2$.
	   \end{remark}
	   The proposition follows from applying Cauchy-Schwarz and the definition of dissipativity, along with \Cref{cor1}; details are provided in \Cref{app1}.  Both \Cref{cor1} and \Cref{cor2} are necessary for the analysis in \Cref{sec:langevin}, the first to show the existence of the Langevin diffusion and the second to show exponential convergence to the stationary distribution. \par
	   The above analysis deals with the population risk, but in reality we are only given $n$ i.i.d samples from $p$.  We have the following result:
	   \begin{proposition}\label{thm2}
	        Let $\mathcal{F}$ be a class of $\mathbb{R}^d$-valued  functions, all of which are $\frac M2$-Lipschitz, bounded coordinate wise by $R > 0$, containing arbitrarily good uniform approximations of $\nabla \log p_{\sigma^2}$ on the ball of radius $R$.  Let $\sigma^2 < \sigma_{max}^2$ and suppose we have $n$ i.i.d samples from $p_{\sigma^2}$, $X_1, \dots, X_n$.  Let
	        \begin{equation}
	            \widehat{s} \in \argmin_{s \in \mathcal{F}} \frac 1n \sum_{i = 1}^n \left|\left| s(X_i) - \nabla \log p_{\sigma^2}(X_i) \right|\right|^2
	        \end{equation}
	        Then with probability at least $1 - 4\delta - Cn e^{- \frac{R^2}{m_{\sigma^2}}}$ on the randomness due to the sample,
	        \begin{equation}
	            \mathbb{E}_{X \sim p_{\sigma^2}} \left[ || \widehat{s}(X) - \nabla \log p_{\sigma^2}(X)||^2\right] \leq C (MR + B)^2 (\log^3 n \cdot \mathfrak{R}_n^2(\mathcal{F}) + \beta_n d)
	        \end{equation}
	        where $C$ is a universal constant, $m_{\sigma^2}$ can be found in \Cref{cor2}, $B$ is from \Cref{growth}, and
	        \begin{equation}
	            \beta_n = \frac{\log \frac 1\delta + \log\log n}{n}
	        \end{equation}
	   \end{proposition}
   		\begin{remark}
			Note while $B$ can be taken to be a bound on the norm of $\nabla \log p$ at the origin, in reality, by replacing $R$ with $2R$ we can take $B$ to be the infimum of $||\nabla \log p$ on the ball of radius $R$ about the origin.  By dissapitivity, if $R$ is sufficiently large, then this constant is small.
   		\end{remark}
	   \begin{remark}
	        In order to have a high probability estimate, we need $C n e^{- \frac{R^2}{m_{\sigma^2}}}$ to be small; as such we see that $R^2 = \Omega(\log n)$.  Thus, up to factors polynomial in $\log n$, we see that the $L^2$ error of the estimate is $\widetilde{O}(\mathfrak{R}_n^2(\mathcal{F}))$.
	   \end{remark}
	   From the equivalence between DSM and DAE loss, we have as an immediate corollary
	   \begin{proposition}\label{cor4}
	        Suppose we are in the setting of \Cref{thm2}.  Let
	        \begin{equation}
	        \label{eq:emp_fit_dae}
	            \widehat{r} \in \argmin_{r \in \mathcal{F}} \frac 1n \sum_{i = 1}^n ||r(X_i + \sigma \xi_i) - X_i||^2
	        \end{equation}
	        where $\xi_i$ are iid standard Gaussians. Then with probability at least $1 - 4\delta - C n e^{- \frac{R^2}{m_{\sigma^2}}}$
	        \begin{equation}
	            \mathbb{E}_{X \sim p_{\sigma^2}} \left[ \left|\left| \frac{(\hat{r}(X) - X)}{\sigma^2}-\nabla \log p_{\sigma^2}(X)\right|\right|^2 \right] \leq \frac{1}{\sigma^4} C (MR + B)^2   (\log^3 n \cdot \mathfrak{R}_n^2(\mathcal{F}) + \beta_n d)
	        \end{equation}
	        where the notation is as in \Cref{thm2}.
	   \end{proposition}
	   \begin{proof}
	        Let $r^*(x) = x + \sigma^2 \nabla \log p_{\sigma^2}(x)$ be the population optimal DAE. 
    	    Using the identical analysis as in \Cref{thm2}, we get that
    	    \begin{equation}
    	        \mathbb{E}_{X \sim p_{\sigma^2}} \left[\left|\left| \widehat{r}(X) - X - \sigma^2 \nabla \log p_{\sigma^2}(X) \right|\right|^2 \right] \leq C (MR + B)^2   (\log^3 n \cdot \mathfrak{R}_n^2(\mathcal{F}) + \beta_n d)
    	    \end{equation}
    	    Dividing by $\sigma^4 > 0$ on both sides of the above inequality concludes the proof.
        \end{proof}
        The proof of \Cref{thm2} is an application of a result on lower isometry in \cite{Rakhlin2017}.  We provide a sketch below, with full details appearing in \Cref{app:empirical}.
        \begin{proof}(Sketch of \Cref{thm2})
            In order to apply the desired result from empirical processes theory, we require our function class to be bounded; while this does not hold, we can provide bounds in high probability.  By the fundamental theorem of calculus and the dissipativity assumption, we have that $p$ has Gaussian tails, as proved in \Cref{subgaussian}.  Thus with probability at least $1 - Cn e^{- \frac{2R^2}{m_{\sigma^2}}}$ all the samples fall in a ball of radius $R$, on which $\mathcal{F}$ is certainly bounded.  Then the minimizer of the population DSM loss is given by $\nabla \log p_{\sigma^2}$ and we are in the well-specified case.  Breaking the error up coordinate-wise, we bound the squared error by the product of the dimension and the largest coordinate-wise error.  Applying \cite[Lemmas 8, 9]{Rakhlin2017} to this coordinate concludes the proof.
        \end{proof}
        
        The generality of \Cref{thm2} with regard to the function class is potentially helpful in fine-grained analysis of how DAEs are used in practice: usually $\mathcal{F}$ is a class of neural networks; combining known results on the complexity of such classes with \Cref{cor3} gives high-probability bounds on the DAE error.  If we consider the special case where $\mathcal{F}$ is the class of Lipschitz functions bounded by $R$ in each coordinate, then we can apply known results on the complexity of this class \cite[Theorem XIII]{Tikhomirov1993} to get a rate of $\widetilde{O}\left( n^{-\frac 2d}\right)$, ignoring factors polynomial in $\log n$. On the other hand, norm-based bounds for Rademacher averages of neural networks in \cite{bartlett2017spectrally,golowich2017size,neyshabur2015norm} can imply the faster $n^{-1}$ rate, as long as the empirical error in \Cref{eq:emp_fit_dae} is small.
\section{The Langevin Process: Approximation and Convergence}\label{sec:langevin}
    In this section, we analyze one section of the homotopy method described above.  As such, we fix an $\eta$ and a $\sigma^2$ and bound $\mathcal{W}_2\left(\mu_k, p\right)$, where we recall that $\mu_k$ is the law of $W_k$, the $k^{th}$ iterate in the discrete Langevin sampling scheme described in the introduction.  We have the following theorem:
    \begin{theorem}\label{thm1}
        Let $d \geq 3$ and suppose that $-\nabla \log p$ is $\frac M2$-Lipschitz and $(m,b)$-dissipative.  Suppose that $\widehat{f}$ is an estimate of $\nabla \log p_{\sigma^2}$ that is also $\frac M2$-Lipschitz and $(m,b)$-dissipative and whose squared error is bounded by $\varepsilon^2$:
        \begin{equation}\label{eq:thm1}
            \mathbb{E}_{X \sim p_{\sigma^2}} \left[\left|\left| \widehat{f}(X) - \nabla \log p_{\sigma^2}(X)\right|\right|^2 \right] \leq \varepsilon^2
        \end{equation}
        If we initialize $W_0$ according to a distribution $\mu_0$ such that $\mathcal{W}_2 \left(\mu_0, p_{\sigma^2}\right) < \infty$ and $\mu_0$ is concentrated in high probability on a ball of radius $R$ and satisfying Assumption \ref{assumption4} for some $\alpha, k_{\alpha}$, and we let $\mu_k$ denote the law of a discrete Langevin sampling scheme with constant step size $\eta$ run for $k$ iterations with score estimate $\widehat{f}$, then, for sufficiently small $\varepsilon > 0$ with $\tau = k \eta$, we have
        \begin{align}\label{eq3}
            &\mathcal{W}_2\left(\mu_k, p\right) \leq \sigma \sqrt{d} + A_{M,d,B}(\eta,\tau) + \sqrt{c_{LS}(\sigma^2) \cdot KL\left( \mu_0|| p_{\sigma^2}\right)} \cdot e^{-  \frac{\tau}{c_{LS}(\sigma^2)}} + C_{\sigma^2, M, R, B, d}(\tau, \varepsilon)
        \end{align}
        where $c_{LS}(\sigma^2)$ can be found in \Cref{prop4},
        \begin{align}
            A_{M,d,B}(\eta, \tau) = C \sqrt{d \eta \tau}e^{\frac{M^2}{2} \tau}
        \end{align}
        and
        \begin{equation}
            C_{\sigma^2, M, R, B, d, \alpha}(\tau, \varepsilon) = C \sqrt{(b + d) \tau} \left(\varepsilon \tau + C ||p_{\sigma^2}||_\infty^{\frac 12 - \frac 1d} e^{\frac{M \sqrt{d}}{4} \tau} \sqrt{\tau} \varepsilon^{\frac 1d}\right)^{\frac 14}
        \end{equation}
        where $B$ is a constant from \Cref{growth} and $C$ depends on $M, B, m, b, \alpha, k_{\alpha}$ and $\mathbb{E}\left[||W(0)||^4\right]$ in both expressions above, with explicit dependence found in \Cref{app1} and \Cref{app3} respectively.
    \end{theorem}
    \begin{remark}
        The $\varepsilon$ in \Cref{eq:thm1} above is controlled in high probability by \Cref{thm2} or \Cref{cor4} as rates depending on the number of samples and the complexity of the function class over which we are optimizing.  Thus, combining \Cref{thm2} instantiated on a given function class and \Cref{thm1} gives explicit high probability bounds on how well the Langevin sampling algorithm with an estimate trained on $n$ samples approximates the population distribution.
    \end{remark}
	\begin{remark}
		Note that the condition that $\varepsilon$ is sufficiently small is included only to give a nice functional form to $C_{\sigma^2, M, R, B, d, \alpha}$ and can be relaxed if we a term with a square root dependence in addition to the fourth root, as seen in \Cref{app3}.
	\end{remark}
    \begin{remark}\label{rmk1}
        While the exponential dependence on $\tau$ in the second and last terms of \Cref{eq3} may seem bad, note that the exponential convergence in the third term tells us that if we want to get $\delta$-close in Wasserstein distance, then $\tau$ needs only be $\mathsf{poly}\left( \log \frac 1\delta\right)$ and so the exponential dependence on $\tau$ is only a polynomial dependence on error. \par
        The exponential growth with respect to dimension is a touch more serious.  If we make more than a Lipschitz assumption on the score function and assume moreover that $|\Delta \log p|\leq C$ then $M \sqrt{d}$ can be replaced by $C$ in the factor under the square root above.  While we leave to future work the job of determining if the $\varepsilon^{\frac 1{2d}}$ factor is tight, we suspect that, without further assumptions, an exponential dependence on dimension cannot be avoided.  In similar work that makes no assumption of convexity, such as \cite{RakhlinRaginsky}, polynomial dependence on dimension is not achieved and there is reason to believe that it cannot be true in general.  Regarding sample complexity, minimax results on Wasserstein estimation (see e.g. \cite{Weed2019,Goldfeld2019}) suggest that exponential dependence in dimension cannot be improved without further assumptions.
    \end{remark}
    The remainder of this section is devoted to a sketch of the proof of \Cref{thm1}; the rigorous proof is relegated to the appendices.  With respect to the approximation of the continuous Langevin diffusion by a discrete process, we apply the standard technique to produce an explicity coupling. Note that our result, in contradistinction to earlier work, requires only that we have access to an estimate of the score that is $L^2(p)$-close, rather than uniformly close, significantly increasing the difficulty of proving the bound.\par
    In order to prove \Cref{thm1}, we consider several intermediate measures.  Let $\nu_t$ be the law of $W(t)$ and $\widehat{\nu}_t$ the law of $\widehat{W}(t)$ at a fixed time $t$.  Then by the triangle inequality we can decompose
    \begin{equation}\label{eq1}
        \mathcal{W}_2\left(\mu_k, p\right) \leq \mathcal{W}_2\left(\mu_k, \widehat{\nu}_{\tau} \right) + \mathcal{W}_2\left( \widehat{\nu}_{\tau}, \nu_{\tau}\right) + \mathcal{W}_2\left( \nu_{\tau}, \nu_{\infty} \right) + \mathcal{W}_2\left(\nu_\infty, p\right)
    \end{equation}
    Note that $\nu_\infty = p_{\sigma^2}$ and so the last error term is controlled by
    \begin{lemma}\label{lem1}
        Let $p$ be a measure on $\mathbb{R}^d$.  Then
        \begin{equation}
            \mathcal{W}_2\left(p, p_{\sigma^2}\right) \leq \sigma \sqrt{d}
        \end{equation}
    \end{lemma}
    \begin{proof}
        This follows immediately from considering the coupling $(X, X + \xi)$ where $X \sim p$ and $\xi \sim g_{\sigma^2}$ is independent of $X$.  The variance of $\xi$ is $\sigma \sqrt{d}$, concluding the proof.
    \end{proof}
    We now bound the other three terms. \par
    The first term in \Cref{eq1} comes from the error introduced by the fact that our sampling algorithm is only an approximation of the continuous Langevin process.  We have as a standard result, with proof in \Cref{app1}:
	\begin{proposition}\label{prop2}
		With the notation as above and assuming that $\widehat{f}$ is $\frac M2$-Lipschitz, we have that there is a constant $C$ depending on $M, m, b$ and linearly on $\mathbb{E}\left[||W(0)||^2\right]$ such that
		\begin{align*}
			\mathcal{W}_2\left(\mu_k, \widehat{\nu}_{\tau} \right)^2 \leq C d \eta \tau e^{M^2 \tau}
		\end{align*}
	\end{proposition}
	Having dispensed with the first and last terms in \Cref{eq1}, we are now ready to tackle the middle terms, the error due to the lack of convergence to the stationary distribution and the error due to the difference between the estimated and population Langevin diffusions.
\subsection{Log-Sobolev Inequalities and Exponential Convergence in Wasserstein Distance}
    In order to deal with convergence in Wasserstein distance to the stationary distribution, we show that for sufficiently small $\sigma^2 \geq 0$, $p_{\sigma^2}$ satisfies a log-Sobolev inequality.  This will also be helpful in bounding the second term in \Cref{eq1}, as we shall see in the following section.  We will use the Lyaponov function criterion as proved in \cite{Cattiaux}.  The key result is the proof that if $-\nabla \log p$ is $(m,b)$-dissipative, then $p$ satisfies a log-Sobolev inequality with a constant bounded in terms of $m$ and $b$, a result proved in \cite{RakhlinRaginsky}.  We have as a translation of Proposition 3.2 from \cite{RakhlinRaginsky}:
    \begin{proposition}\label{prop4}
        Let $-\nabla \log p$ be $\frac M2$-Lipschitz and $(m,b)$-dissipative.  Then for $0 \leq \sigma^2 \leq \sigma_{max}^2$, the smoothed distribution $p_{\sigma^2}$ satisfies a log-Sobolev inequality with constant
        \begin{equation}
            c_{LS}(\sigma^2) \leq \frac{8M}{m_{\sigma^2}^2} + \frac 2M + c_P \left(2 + \frac{6M}{m_{\sigma^2}} \left(b+d \right)\right)
        \end{equation}
        Where $m_{\sigma^2}$ is the constant appearing in \Cref{cor2} and
        \begin{equation}
            c_P(\sigma^2) \leq \frac{2}{m_{\sigma^2}(d+b)} \left(1 + C (d + b)^2 e^{\frac{8(M + B)(d + b)}{m_{\sigma^2}}}\right)
        \end{equation}
        where the constant $B$ appears in \Cref{growth}.
    \end{proposition}
	\begin{remark}
		Note that without further assumptions, the exponential dependence on dimension is unavoidable.  This is because were this not the case, we would mix in polynomial time, which would allow for general nonconvex optimization in polynomial time, even though the problem is NP-hard.  Thus at the very least, the above bound has the tightest functional dependence on dimension.
	\end{remark}
    We replicate the proof used in \cite{RakhlinRaginsky} in detail, in \Cref{app:sobolev}.  With the log-Sobolev constant established, the convergence in Wasserstein distance is immediate:
    \begin{proposition}\label{prop5}
        Let $-\nabla \log p$ be $\frac M2$-Lipschitz and $(m,b)$-dissipative.  Then for $0 \leq \sigma^2 \leq \sigma_{\max}^2$, if $\nu_t$ is the law of $W(t)$, then
        \begin{equation}
            \mathcal{W}_2\left(\nu_t, p_{\sigma^2} \right) \leq \sqrt{c_{LS}(\sigma^2) \cdot KL(\nu_t , p_{\sigma^2}) } \leq \sqrt{c_{LS}(\sigma^2) \cdot KL(\mu_0 , p_{\sigma^2}) } \cdot e^{- \frac t{c_{LS}(\sigma^2)}}
        \end{equation}
    \end{proposition}
    The proof of this result is well known given the Otto-Villani theorem \cite[Theorem 1]{Otto2000} and the exponential convergence of Langevin in $KL(\cdot , p_{\sigma^2})$ under a log-sobolev inequality; see \cite{Bakry} for details.

\subsection{Running a Langevin Diffusion with a Score Estimator}\label{subsec:hatdiffusion}
    The second term in \Cref{eq1} is the error due to the difference between running a continuous Langevin diffusion with drift $\widehat{f}$ and that with drift $\nabla \log p_{\sigma^2}$.  The details are in \Cref{app3}, but a sketch of the argument is below.  Let $\widehat{\nu}_t$ be the law of $\widehat{W}(t)$ and let $\nu_t$ be the law of $W(t)$ at a fixed time $t$.   We recall that the classic theorem of Bolley and Villani (found in \cite{Bolley2005}) tells us that for all $t > 0$,
    \begin{equation}
        \mathcal{W}_2\left( \widehat{\nu}_t, \nu_t \right) \leq C \left(\sqrt{KL(\nu_t ,\widehat{\nu}_t}) + \left(\frac{KL(\nu_t, \widehat{\nu}_t)}{2}\right)^{\frac 14} \right)
    \end{equation}
    where
    \begin{equation}
    	C = 2 \inf_{\alpha > 0} \left(\frac 1\alpha\left(\frac 32 + \log \mathbb{E}\left[e^{\alpha ||\widehat W(t)||^2} \right]\right) \right)^{\frac 12}
    \end{equation}
    and
    \begin{equation}
        KL(\nu_t, \widehat{\nu}_t) = \mathbb{E}_{\nu_t}\left[\log \frac{d \nu_t}{d \widehat{\nu}_t} \right]
    \end{equation}
    is the relative entropy. The constant $C$ is controlled in \Cref{cor:bolley} in \Cref{app3} using the identical technique as that applied in \cite{RakhlinRaginsky}.  Thus it suffices to bound the relative entropy.\par
    In order to compute the relative entropy between the laws of two diffusions with the same noise but different drifts, we can apply Girsanov's theorem and take expectations, as in \cite{Dalalyan2014,Dalalyan2012,RakhlinRaginsky}.  In particular, this tells us that
    \begin{proposition}\label{prop6}
        Let $W(t), \widehat{W}(t)$ be as above and assume that $\widehat{f}, \nabla \log p_{\sigma^2}$ is $\frac M2$-Lipschitz.  Then
        \begin{equation}
            KL\left(\nu_t , \widehat{\nu}_t\right) = \mathbb{E}\left[\frac 12\int_0^t \left|\left| \nabla \log p_{\sigma^2}( W(s)) - \widehat{f}( W(s))\right|\right|^2\right]
        \end{equation}
    \end{proposition}
    Thus it suffices to control this last quantity.  If we had that $\widehat{f}$ were uniformly close to $\nabla \log p$ or even that we were close in the $L^2(W(t))$ sense, we would be done; unfortunately, we only have that the estimate and the score are close in the sense of $L^2(p_{\sigma^2})$ and so it is not a priori obvious that the above integral can be controlled.  In order to get around this, we use the concept of local time and expected occupation density, as seen in \cite{Karatzas,Geman1980}, as well as bounds on the transition density of a diffusion.
    \begin{proposition}\label{prop3}
        Let $\widehat{f}$ be an estimator of $\nabla \log p_{\sigma^2}$ such that both the estimator and the score are $\frac M2$-Lipschitz and $(m,b)$-dissipative and the error is bounded:
        \begin{equation}
            \mathbb{E}_{X \sim p_{\sigma^2}} \left[\left|\left| \widehat{f}(X) - \nabla \log p_{\sigma^2}(X)\right|\right|^2 \right] \leq \varepsilon^2
        \end{equation}
        Suppose we initialize $W(0) \sim \mu_0$ which is concentrated in high probability on a ball of radius $R$.  Then with high probability under the randomness due to initialization at $x$,
        \begin{align}
		\mathbb{E}_x\left[\int_0^t \left|\left|\nabla \log p(W(s)) - \widehat{f}(W(s))\right|\right|^2 d s\right] \leq \varepsilon t + C ||p_{\sigma^2}||_{L^\infty}^{\frac12 - \frac 1d}e^{\frac{M \sqrt{d}}{4} t} \sqrt{t} \varepsilon^{\frac 1d}
        \end{align}
        for a constant $C$ depending on $M, B, m, b, R$ and $\mathbb{E}\left[||W(0)||^4\right]$, whose explicit dependence is given in \Cref{app3}.
    \end{proposition}
    A full proof can be found in \Cref{app3}, but we provide a sketch without details.
    \begin{proof}(Sketch)
        In order to bound the desired expected value, we consider the set $U \subset \mathbb{R}^d$ where $\left|\left| \widehat{f}(x) - \nabla \log p_{\sigma^2}(x)\right|\right| > \varepsilon$.  We can break the integral into the times when $W(s) \in U$ and the times when $W(s) \not\in U$.  In the latter case we can apply a uniform bound of $\varepsilon t$.  In the former case, we can apply Cauchy-Schwarz to bound this term by the square root of the product of the fourth moment of $||\widehat{f}((W(s)) - \nabla \log p_{\sigma^2}(W(s)||$ and the expected amount of time that $W(s)$ spends in $U$.  The former follows easily from the Lipschitz bounds of \Cref{growth} and the moment bound of \Cref{moments}.  The second comes from a bound on expected occupation times, which are time integrals of the transition density of the Langevin diffusion.  We bound the Radon-Nikodym derivative of these transition densities with respect to the population distribution with the Girsanov theorem, and then integrate with respect to time.  Finally, we apply a generalization of a famous inequality of Hardy and Littlewood, giving the result.
    \end{proof}
    Combining \Cref{prop3} with the theorem of Bolley and Villani and \Cref{exponential} yields
    \begin{corollary}\label{cor3}
        Let $\widehat{f}$ be an estimator of $\nabla \log p_{\sigma^2}$ such that both the estimator and the score are $\frac M2$-Lipschitz and $(m,b)$-dissipative.  Suppose further that the error is bounded:
        \begin{equation}
            \mathbb{E}_{X \sim p_{\sigma^2}} \left[\left|\left| \widehat{f}(X) - \nabla \log p_{\sigma^2}(X)\right|\right|^2 \right] \leq \varepsilon^2.
        \end{equation}
        Suppose we initialize $W(0) \sim \mu_0$ which is concentrated in high probability on a ball of radius $R$ and satisfies Assumption \ref{assumption4}.  Then, with high probability under the randomness due to initialization at $x$, , we have
        \begin{equation}
		\mathcal{W}_2\left(\widehat{\nu}_t, \nu_t\right) \leq C \sqrt{(b + d)t}\left( \left(\varepsilon t + C ||p||_\infty^{\frac 12 - \frac 1d} e^{\frac{M \sqrt{d}}{4} t} \sqrt{t} \varepsilon^{\frac 1d}\right)^{\frac 12} + \left(\varepsilon t + C ||p||_\infty^{\frac 12 - \frac 1d} e^{\frac{M \sqrt{d}}{4} t} \sqrt{t} \varepsilon^{\frac 1d}\right)^{\frac 14}\right)
        \end{equation}
        where $C$ depends on $\alpha, k_{\alpha}$ in Assumption \ref{assumption4} with the explicit dependence given in \Cref{app3}.
    \end{corollary}
    We are finally ready to prove \Cref{thm1}:
    \begin{proof}(Proof of \Cref{thm1})
        Using the triangle inequality, we need to bound each of the four terms in \Cref{eq1}.  These terms are bounded in \Cref{prop2}, \Cref{cor3}, \Cref{prop5}, and \Cref{lem1}.  If $\varepsilon > 0$ is sufficiently small then
        \begin{equation}
			\left(\varepsilon t + C ||p||_\infty^{\frac 12 - \frac 1d} e^{\frac{M \sqrt{d}}{4} t} \sqrt{t} \varepsilon^{\frac 1d}\right)^{\frac 12} \leq \left(\varepsilon t + C ||p||_\infty^{\frac 12 - \frac 1d} e^{\frac{M \sqrt{d}}{4} t} \sqrt{t} \varepsilon^{\frac 1d}\right)^{\frac 14}
        \end{equation}
        and thus up to a factor of 2 we can keep only the larger term.  This concludes the proof.
    \end{proof}
\section{Homotopy and Annealing}
    The above section and the discussion surrounding \Cref{thm1} focuses on one leg of the homotopy; this section uses \Cref{thm1} to analyze the effect that the homotopy method and annealing the DAE has on the Wasserstein distance between the sample and the population distribution. \par
     We consider the following sampling scheme.  For fixed $k$, we fix a sequence $\left\{ \left( \eta_i, \sigma_i^2\right) | 1 \leq i \leq N \right\}$ where $\eta_i, \sigma_i^2$ are both decreasing in $i$.   We train a DAE with variance parameter $\sigma_i^2$, $r_i$, and set $s_i(x) = \frac 1{\sigma_i^2}(r_i(x) - x)$  Then we initialize $W^{(1)} \sim \mu_0$ and evolve with $\widehat{f} = s_1$ with step size $\eta_1$ for $k$ iterations.  Then, we use warm restarts and for $2 \leq i \leq N$, we evolve a Langevin sampling scheme $W^{(i)}$ with $\widehat{f} = s_i$, step size $\eta_i$, and $W_0^{(i)} = W_k^{(i-1)}$.  Our final sample is $W_k^{(N)}$. \par
    The decomposition in \Cref{eq1} provides clues as to why the homotopy method described above speeds up the Langevin sampling.  Consider, first, the effect that $\eta$ has on the decomposition.  According to \Cref{thm1}, we have that the error is bounded by
    \begin{align}\label{eq4}
            &\mathcal{W}_2\left(\mu_k, p_{\sigma^2}\right) \leq \sigma \sqrt{d} + A_{M,d,B}(\eta,\tau) + \sqrt{c_{LS}(\sigma^2) \cdot KL\left( \mu_0, p_{\sigma^2}\right)} e^{- \frac{\tau}{c_{LS}(\sigma^2)}} + C_{\sigma^2, M, R, B, d, \alpha}(\tau, \varepsilon)
    \end{align}
    As $\eta$ increases, since $\tau = k \eta$, we see that $\tau$ increases and we note that both $A(\eta, \tau)$ and $C(\tau,\varepsilon)$ increase, but the third term decreases.  Thus with all other constants fixed, the optimal $\eta$ can be determined as $\eta_{opt} > 0$.  \par
    With regard to score estimation, the greater value of $\sigma^2$ makes it easier to estimate the score.  Consider that, in our regime, we are training a DAE and then using the transformation in \Cref{cor3} to plug the score estimate into \Cref{thm1}.  With a fixed number of samples, doing the proof of \Cref{cor4} in reverse, we note that if the DAE has squared error $\widetilde{\varepsilon}^2$, then the score estimate has squared error $\frac 1{\sigma^4} \widetilde{\varepsilon}^2$. Thus, if we fix $\widetilde{\varepsilon}^2$ as the achievable error of a DAE trained on $n$ samples, then as $\sigma^2$ increases, $\varepsilon^2$ decreases and thus so, too, does $C_{\sigma^2, M, R, B, d, \alpha}(\tau, \varepsilon)$. \par
    The effect of $\sigma^2$ on the log-Sobolev constant remains a bit more mysterious from a rigorous point of view.  While \Cref{cor2} provides a bound on $c_{LS}(\sigma^2)$, it is likely not tight, as it amounts to a `worst-case' analysis of the effect that the Gaussian smoothing has on the population distribution, using a crude argument involving Cauchy-Schwarz.  If we make further assumptions, these results can be tightened.  For example, if we suppose that $p$ has compact support, then \cite{Bardet} gives a bound on the log-Sobolev constant of the smoothed distribution that decreases with larger variance, thereby accelerating the convergence of the Langevin diffusion to its stationary distribution; while their bound is dimension-dependent, they suggest that future work may lose this handicap.  Thus, in this case, just as there exists an $\eta_{opt}$ that minimizes the right hand side of \Cref{eq4}, there is too such a $\sigma_{opt}^2$.  We leave to future work the task of better controlling $c_{LS}(\sigma^2)$ in the general case.  Thus the annealing and the homotopy method combine to provide a form of dynamic optimization of the upper bound of \Cref{thm1}, significantly decreasing the error and simultaneously speeding up the naive Langevin sampling that does not involve homotopy or annealing.  The above is empirically indicated by the success of the annealed score matching in \cite{Ermon}; as we saw in \Cref{prop1}, though, the annealed score matching is equivalent to DAE loss, so the empirical success in one area transfers to the other \emph{mutatis mutandis}.

\section{Conclusion and Further Directions}
    We have provided rigorous justification above for two empirical approaches that have recently generated excitement: the use of score estimators to run Langevin sampling and the homotopy method of \cite{Ermon}.  While the bounds in \Cref{thm1,thm2} allow for high probability guarantees with finite samples, there is a question of tightness.  First, the dependence on the dimension of the bound in \cref{eq3} is potentially suboptimal: it is possible that a more detailed analysis of the argument proving \Cref{prop3} would substantially improve this dependence on the dimension.  Second, while we have proved that the smoothed distribution $p_{\sigma^2}$ satisfies a log-Sobolev inequality, we almost certainly do not have the optimal log-Sobolev constant.  In fact, as discussed in the above section, while our bound on this constant gets worse with more noise, it is likely that the log-Sobolev constant actually improves with increased variance, which would explain fully the benefits of the annealing of DSM estimators that is so successful in \cite{Ermon}.  Third, we could likely improve the algorithm by using lower variance estimators of the score function and a higher-order method for approximating the continuous Langevin diffusion.  There has been some recent theoretical work in this direction (see, for example, \cite{Li2019}) and we suspect that, especially in practical application, this would considerably accelerate the algorithm. \par
    While the unconditional generative modeling studied in this work is certainly interesting in its own right, practitioners tend to focus on the benefits of conditional generative modeling, i.e., where there are two variables $x, y$ and we wish to input $y$ and generate samples from the conditional distribution $p(x|y)$.  It is highly likely that, given the right conditions on the joint distribution of $x$ and $y$ to ensure a uniformity in $y$ to the dissipativity and Lipschitz nature of the conditional distribution, many of the results above could be extended to this regime.

\section*{Acknowledgements}
	We would like to acknowledge the support from the MIT-IBM Watson AI Lab.  In addition, this material is based upon work supported by the National Science Foundation Graduate Research
	Fellowship under Grant No. 1122374.  We would also like to thank  Thibaut Le Gouic and Sinho Chewi for pointing out an error in an earlier draft.

\bibliographystyle{alpha}
\bibliography{references}

\appendix

\section{Empirical Processes and Proving Proposition \ref{thm2}}\label{app:empirical}
    We briefly sketch a few definitions from the theory of empirical processes.   Denote by $\mathbb{E}$ as expectation with respect to $p_{\sigma^2}$ and by $\widehat{\mathbb{E}}$ as expectation with respect to the empirical measure.  We have already defined the Rademacher complexity above.  Given an $r > 0$, a function class $\mathcal{F}$, and a sample of $n$ points $X_1, \dots, X_n \in \mathbb{R}^d$, we let
    \begin{equation}
        \mathcal{F}[r, S] = \left\{ f \in \mathcal{F} : \frac 1n \sum_{i = 1}^n f(X_i) \leq r \right\}
    \end{equation}
    We call a function $\phi_n: [0, \infty) \to \mathbb{R}_{\geq 0}$ an upper function for $\mathcal{F}$ if for all $r > 0$,
    \begin{equation}
        \sup_{S \subset \left(\mathbb{R}^d\right)^n} \widehat{\mathfrak{R}}_n(\mathcal{F}[r, S], S) \leq \phi_n(r)
    \end{equation}
    We define the  (nonunique) localization radius $r^*$ as an upper bound on the maximal solution to the equation $\phi_n(r) = r$.  We recall two results:
    \begin{lemma}\label{lem:rak1}\cite[Lemma 8.i]{Rakhlin2017}
        For any class $\mathcal{F}$ of real valued functions with image in the interval $[0,1]$, with $n \geq 2$, we may take as localization radius of $\mathcal{G} = \left\{ (f - f')^2| f \in \mathcal{F}\right\}$
        \begin{equation}
            r^* = C \log^3(n) \mathfrak{R}_n(\mathcal{F})^2
        \end{equation}
        where $C$ is a universal constant.
    \end{lemma}
    \begin{lemma}\label{lem:rak2}\cite[Lemma 9]{Rakhlin2017}
        For any class $\mathcal{F}$ of real valued functions with image contained in the unit interval, and $\delta > 0$, with probability at least $1 - 4 \delta$, if $X_1, \dots, X_n \sim p$ i.i.d., for all $f, f' \in \mathcal{F}$ 
        \begin{equation}
            \mathbb{E}[(f - f')^2] \leq 2\widehat{\mathbb{E}} (f - f')^2 + C(r^* + \beta)
        \end{equation}
        where $r^*$ is the localization radius of $\mathcal{G} = \{(f - f')^2| f,f' \in \mathcal{F}\}$ and
        \begin{equation}
            \beta = \frac{\log \frac 1\delta + \log\log n}{n}
        \end{equation}
        and $C$ is a universal constant.
    \end{lemma}
    With these results, we can prove \Cref{thm2}:
    \begin{proof}(Proof of \Cref{thm2})
        We prove for the case $\sigma^2 = 0$, as we rely only on the dissipativity and Lipschitz assumptions; thus \Cref{cor1} and \Cref{cor2} allow us to apply the same argument with slightly different constants. \par
        By the fundamental theorem of calculus and the dissipativity assumption, we may invoke \Cref{subgaussian} to get that for sufficiently large $R > 0$ that  there is a constant $C$ such that the probability that $X \sim p$ has norm bigger than $R$ is bounded above by $C e^{- \frac{mR^2}{4}}$.  By a union bound, with probability at least $1 - Cn e^{- \frac{R^2}{m}}$ all of the samples lie in the ball bounded by $R$.  On this event then, the norm of $\nabla \log p$ evaluated on the data is bounded by $M R + B$ by \Cref{growth}.  Thus up to a factor of $M R + B$ we are in the situation of \Cref{lem:rak1} and \Cref{lem:rak2} with $\nabla \log p \mathbf{1}_{||X|| \leq R} \in \mathcal{F}$.  Let $s^* = \nabla \log p$ and 
        \begin{align}
            \widehat{s} &\in \argmin_{s\in\mathcal{F}} \widehat{\mathbb{E}}[||s(X) - s^*(X)||^2] .
        \end{align}
        Clearly,
        \begin{equation}
            \widehat{\mathbb{E}}[||\widehat{s}(X) - s^*(X)||^2] = 0
        \end{equation}
        since $s^*\in\mathcal{F}$ gives that
        \begin{equation}
            0 \leq \widehat{\mathbb{E}}[||\widehat{s}(X) - s^*(X)||^2] \leq \widehat{\mathbb{E}}[||s^*(X) - s^*(X)||^2] = 0
        \end{equation} 
        We have
        \begin{equation}
             \mathbb{E}[||\widehat{s}(X) - s^*(X)||^2] = \sum_{i=1}^d \mathbb{E}[(\widehat{s}(X)_i - s^*(X)_i)^2 ]
        \end{equation}
        where the subscript $i$ denotes coordinates.  Applying \Cref{lem:rak2} yields
        \begin{align}
            \mathbb{E}[(\widehat{s}(X)_i - s^*(X)_i)^2] &\leq 2 \widehat{\mathbb{E}}[(\widehat{s}(X)_i- s^*(X)_i)^2] + C (MR + B)(r^*_i + \beta) 
        \end{align}
        where $r^*_i$ is the localization radius for the coordinate restriction $\mathcal{F}_i$.
        Applying \Cref{lem:rak1} to bound $\sum r^*_i$ and noting that we are off by a factor of $(MR + B)$ from the result in \cite{Rakhlin2017}, gives that
        \begin{equation}
            \sum_{i=1}^d r^*_i \leq (MR + B) C \log^3(n) \cdot \sum_{i=1}^d \mathfrak{R}_n(\mathcal{F}_i)^2 \leq (MR + B) C \log^3(n) \cdot  \mathfrak{R}_n(\mathcal{F})^2
        \end{equation}
        which concludes the proof.
    \end{proof}

\section{The log-Sobolev Constant and Convergence to the Stationary Distribution}\label{app:sobolev}
    Background on log-Sobolev inequalities can be found in \cite{Bakry}.  Recall that the generator of the process $W(t)$ is given by a second order differential operator $\mathcal{L}$ acting on a test function $u$ by
    \begin{equation}
        \mathcal{L}u = \Delta u + \langle \nabla \log p_{\sigma^2}, \nabla u \rangle
    \end{equation}
    We call the Dirichlet form evaluated on a function $f$:
    \begin{equation}
        \mathcal{E}(f) = \int ||\nabla f||^2 p_{\sigma^2} d x
    \end{equation}
    Note that $W(t)$, if $\nabla \log p_{\sigma^2}$ is Lipschitz, has a unique invariant distribution of $p_{\sigma^2}$.  We say that $p_{\sigma^2}$ satisfies a Poincar{\'e} inequality  with constant $c_P$ if for all measures $\mu \ll p_{\sigma^2}$, we have
    \begin{equation}
        \int \left|\frac{d \mu}{d p_{\sigma^2}} - 1 \right|^2 p_{\sigma^2} \leq c_P \mathcal{E}\left( \sqrt{\frac{d \mu}{d p_{\sigma^2}}}\right)
    \end{equation}
    We say that $p_{\sigma^2}$ satisfies a log-Sobolev inequality with constant $c_{LS}$ if for all $\mu \ll p_{\sigma^2}$, we have
    \begin{equation}
        KL(\mu, p_{\sigma^2}) \leq c_{LS} \mathcal{E}\left(\sqrt{\frac{d \mu}{d p_{\sigma^2}}} \right)
    \end{equation}
    where $KL(\mu, \nu) = \mathbb{E}_\mu \left[\log \frac{d \mu}{d \nu} \right]$ is the relative entropy.  If $\nabla \log p$ were strongly concave, then there are well-known bounds on the log-Sobolev constant; as we assume no such convexity, we, as in \cite{RakhlinRaginsky}, use a dissipativity condition and the Lyaponov function criteria found in \cite{Guillin,Cattiaux}, presented in the following two theorems:
    \begin{theorem}(\cite{Guillin})
        Let $V: \mathbb{R}^d \to [1,\infty)$ be a real valued function, and let $\mathcal{L}$ be the generator of the diffusion $W$ with stationary distribution $p$.  If there are constants $\lambda_1, \lambda_2, R > 0$ such that
        \begin{equation}
            \frac{\mathcal{L}V(x)}{V(x)} \leq - \lambda_1 + \lambda 2 \mathbf{1}_{B_R}(x)
        \end{equation}
        Then $p$ satisfies a Poincar{\'e} inequality with constant
        \begin{equation}
            c_P \leq \frac{1}{\lambda_1} \left(1 + C \lambda_2 R^2 e^{ \osc_R(\log p)} \right)
        \end{equation}
        where $C > 0$ is a universal constant and for a continuous, real-valued function $f$, we let $\osc_R(f) = \max_{B_R} f - \min_{B_R} f$.
    \end{theorem}
    \begin{theorem}(\cite{Cattiaux})
        Suppose that $p$ is a measure such that $\nabla^2 \log p \succeq - K I_d$ for $K \geq 0$ in the sense of matrices and that $p$ satisfies a Poincar{\'e} inequality with constant $c_P$.  Further suppose that there is a Lyaponov function $V: \mathbb{R}^d \to [1,\infty)$ such that
        \begin{equation}
            \frac{\mathcal{L}V(x)}{V(x)} \leq \kappa - \gamma ||x||^2
        \end{equation}
        Then $p$ satisfies a log-Sobolev inequality with constant
        \begin{equation}
            c_{LS} \leq \frac{2K}{\gamma} + \frac 2K + c_P\left(2 + \frac {2K}\gamma \left( \kappa + \gamma \mathbb{E}_{X \sim p}\left[ ||X||^2\right]\right)\right) 
        \end{equation}
    \end{theorem}
    Further discussion regarding both theorems can be found in the appendix of \cite{RakhlinRaginsky}.  With this in hand, we are ready to prove \Cref{prop4}:
    \begin{proof}(Proof of \Cref{prop4})
        Given that we rely on the $(m,b)$-dissipativity and the $\frac M2$-Lipschitz of $p_{\sigma^2}$, it suffices to prove in the case of $\sigma^2 = 0$ and apply \Cref{cor2}.  We apply in sequence the theorems from \cite{Guillin,Cattiaux}.  Consider the following Lyapanov function: $V(x) = e^{\frac m4 ||x||^2}$.  We compute:
        \begin{align}
            \frac{\mathcal{L}V}{V} &= \Delta\left( \frac m4 ||x||^2 \right) + \left|\left| \nabla\left( \frac{m}{4} ||x||^2\right) \right|\right| + \left\langle \frac m4 x, \nabla \log p(x) \right\rangle \\
            &= \frac{md}{2} + \frac{m^2}{4} ||x||^2 - \frac m2 \langle x, -\nabla \log p(x) \rangle \\
            &\leq \frac m2\left( d + b\right) - \frac{m^2}{4} ||x||^2
        \end{align}
        by the dissapitivity assumption.  Let
        \begin{align}
            R^2 = \frac{4\left(d + \frac b2\right)}{m} && \lambda_1 = m(d + b) && \lambda_2 = \frac m2\left(d + b\right)
        \end{align}
        in the theorem of \cite{Guillin}.  Then we see that by \Cref{growth} and the fundamental theorem of calculus, we may choose $B$ such that
        \begin{equation}
            \osc_R(\log p) \leq (M+B)R^2 + B
        \end{equation}
        Thus we have
        \begin{equation}
            c_P \leq \frac{2}{m(d + b)} \left( 1 + C  (d + b)^2 e^{\frac{8(M+B)(d+b)}{m}}\right)
        \end{equation}
        Now, we may apply the theorem of \cite{Cattiaux} with the same Lyapanov function.  Note that as $\nabla \log p$ is $\frac{M}{2}$-Lipschitz, we have $\nabla^2 \log p \geq - M I_d$.  Then if we set
        \begin{align}
            \kappa = \frac m2 \left(d + b\right) && \gamma = \frac {m^2}4
        \end{align}
        then the above computation with the Lyaponov function shows that the assumptions of the theorem of \cite{Cattiaux} hold.  In order to conclude, we need to bound $\mathbb{E}_p[||X||^2]$.  Let $W(t)$ be the Langevin diffusion initialized on a measure with finite second moment.  Then by \Cref{fourthmoment}, we have
        \begin{align}
            \mathbb{E}_p[||X||^2] = \lim_{t \to \infty} \mathbb{E}[||W(t)||^2] \leq \lim_{t \to \infty} \mathbb{E}\left[||W(0)||^2\right] e^{-2mt} + \frac{b+d}{m}\left(1 - e^{-2mt}\right) = \frac{b+d}{m}
        \end{align}
        The result follows.
    \end{proof}
\section{Bounding the Distance between $\widehat{W}(t)$ and $W(t)$}\label{app3}
    In this appendix, we provide a detailed account of the results appearing in \Cref{subsec:hatdiffusion}, in particular a proof of  the key proposition, \Cref{prop3}.  We assume that $\widehat{f}$ is Lipschitz and, for the sake of convenience, we assume that $\widehat{f}$, $\nabla\log p_{\sigma^2}$ have the same Lipschitz constant $\frac M2$.  Again we let $W(t)$ evolve according to the continuous Langevin process, and we consider the process $\widehat{W}(t)$, the unique solution
	\begin{align*}
		d \widehat{W}(t) =  \widehat{f}\left(\widehat{W}(t)\right) d t + \sqrt{ 2} d B(t) && \widehat{W}(0) = W(0)
	\end{align*}
	First, we need to control the Wasserstein distance in terms of the relative entropy.  In order to do this, we apply the following result of Bolley and Villani: 
	\begin{theorem}[Corollary 2.3 from \cite{Bolley2005}] \label{bolley}
		Let $X \sim \mu$ be a random variable in $\mathbb{R}^d$ and suppose that
		\begin{equation}
			C = 2 \inf_{\alpha > 0} \left(\frac 1\alpha \left(\frac 32 + \log \mathbb{E}\left[e^{\alpha ||X||^2}\right]\right)\right)^{\frac 12} < \infty
		\end{equation}
		Then for all $\nu \ll \mu$
		\begin{equation}
			W_2(\nu, \mu) \leq C \left(KL(\nu, \mu)^{\frac 12} + \left(\frac{KL(\nu, \mu)}{2}\right)^{\frac 14}\right)
		\end{equation}
	\end{theorem}
	As a corollary, we have
	\begin{proposition}\label{cor:bolley}
		Suppose that $\widehat{f}$ is $(m, b)$-dissipative and $\frac M2$-Lipschitz, and that we are in the setting of Assumption \ref{assumption4} for some $\alpha < m$.  Then if $\nu_t$ is the law of $W(t)$ and $\mu \ll \nu_t$ then
		\begin{equation}
			W_2(\mu, \widehat \nu_t) \leq 2 \left(\frac{3}{2\alpha} + \frac{k_\alpha}{\alpha} + 2 (b + d)t\right)^{\frac 12} \left(\sqrt{KL(\mu, \widehat \nu_t)} + \left(\frac{KL(\mu, \widehat \nu_t)}{2}\right)^{\frac 14} \right)
		\end{equation}
		In particular, if $KL(\mu, \widehat \nu_t) \leq 1$ then
		\begin{equation}
			W_2(\mu, \widehat \nu_t) \leq 32 \left(\frac{3}{2\alpha} + \frac{k_\alpha}{\alpha} + 2 (b + d)t\right)^{\frac 12} (KL(\mu, \widehat \nu_t))^{\frac 14} \leq C_\alpha \sqrt{(b + d)t}KL(\mu,\widehat  \nu_t)^{\frac 14}
		\end{equation}
	\end{proposition}
	\begin{proof}
		The second statement follows immediately from the first and the fact that for $c \leq 1$, $c^{\frac 12} \leq c^{\frac 14}$.  The first statement follows from \Cref{bolley} if we can bound the constant $C$.  This bound follows immediately from \Cref{exponential}, concluding the proof.
	\end{proof}
	Thus we are left with bounding the relative entropy between $\nu_t$, the law of $W(t)$ and $\widehat{\nu}_t$ the law of $\widehat{W}(t)$.  We follow the method used in \cite{RakhlinRaginsky,Bubeck2018} of the application of the Girsanov theorem.  In particular, we restate and prove \Cref{prop6}:
    \begin{proposition}
        Let $W(t), \widehat{W}(t)$ be as above and assume that $\nabla \log p_{\sigma^2}$ is $\frac M2$-Lipschitz.  Suppose further that $W(0) = \widehat{W}(0) \sim \mu_0$ such that
        \begin{equation}
            \mathbb{E}_{X \sim \mu_0} \left[ e^{2 M^2 ||X||^2}\right] < \infty
        \end{equation}
        Then
        \begin{equation}
            KL\left(\nu_t, \widehat{\nu}_t \right) = \mathbb{E}\left[\frac 14 \int_0^t \left|\left| \nabla \log p_{\sigma^2}( W(s)) - \widehat{f}( W(s))\right|\right|^2\right]
        \end{equation}
    \end{proposition}
    \begin{proof}
        We apply the version of Girsanov's theorem found in Theorem 7.20 in \cite{Lipster}.  Note that 
        \begin{equation}
            \mathbb{E}\left[\exp\left( \frac 14 \int_0^t ||\widehat{f}(\widehat W(s)) - \nabla \log p_{\sigma^2}(\widehat W(s)||^2 d s\right)  \right] \leq \mathbb{E}\left[ \exp \left( 2M^2 \int_0^t ||\widehat W(t)||^2 d s \right) \right] < \infty
        \end{equation}
        by \Cref{exponential}.  Thus by Girsanov's theorem, we have that
        \begin{align}
            \frac{d \nu_t}{d \widehat{\nu}_t} = \exp\left(\frac{1}{\sqrt{2}}\int_0^t (\widehat{f}(\widehat W(s)) - \nabla \log p_{\sigma^2}(\widehat W(s))) d B_s - \frac 14 \int_0^t ||\widehat{f}(\widehat W(s)) - \nabla \log p_{\sigma^2}(\widehat W(s))||^2 d s\right)
        \end{align}
        and so
        \begin{align}
            KL(\nu_t, \widehat{\nu}_t) &= \mathbb{E}_{ \nu_t}\left[\log \frac{\nu_t}{d \widehat{\nu}_t}  \right] \\
            &= \mathbb{E}\left[\-\int_0^t (\widehat{f}( W(s)) - \nabla \log p_{\sigma^2}( W(s))) d B_s + \frac 14 \int_0^t ||\widehat{f}(W(s)) - \nabla \log p_{\sigma^2}(W(s))||^2 d s \right]
        \end{align}
        But the first term is a real martingale by Novikov's condition so has expectation zero, yielding the result.
    \end{proof}
	Thus it suffices to bound
	\begin{align*}
		\mathbb{E}\left[\int_0^t || \widehat{f}( W(s)) - \nabla\log p_{\sigma^2}( W(s))||^2 d s \right]
	\end{align*}
	We first need the following identity:
	\begin{lemma}\label{localtime}
		Let $U \subset \mathbb{R}^d$ be measurable and let $W(s)$ be as above.  Let $\pi_s(x, y)$ be the transition density of $W(s)$.  Then if $\mathbb{E}_x[\cdot]$ denotes expectation with respect to the measure associated with $W(s)$ started at $W(0) = x$, then
		\begin{align}
			\mathbb{E}_x\left[\int_0^t \mathbf{1}_U(W(s)) d s\right] = \int_U\left(\int_0^t \pi_s(x,y) d s\right) d y
		\end{align}
	\end{lemma}
	\begin{proof}
		Following the proof of \cite[Theorem 3.32]{Peres} and appealing to \cite{Geman1980} for justification in the case that $W(t)$ is not just Brownian motion, we have
		\begin{align}
			\mathbb{E}_x\left[\int_0^t \mathbf{1}_U(W(s)) d s\right] = \int_0^t \mathbb{E}_x\left[\mathbf{1}_U(W(s))\right] d s = \int_0^t \int_U \pi_s(x, y) d y d s = \int_A\left(\int_0^t \pi_s(x,y) d s\right) d y
		\end{align}
		by Fubini's theorem and the definition of the transition density.
	\end{proof}
	In order to construct a bound on the relevant integral, we need a bound on $\pi_t(x,y)$.  In order to construct a Gaussian bound, we adapt an argument of \cite{Downes}:
	\begin{lemma}\label{parabolicpde}
		For fixed $x$, let $W(t)$ evolve as above and let $W(0) = x$.  If $\nabla \log p_{\sigma^2}$ is $\frac M2$-Lipschitz, $\pi_t(x,y)$ is the transition density of $W(t)$ for $y \in \mathbb{R}^d$, then we have
		\begin{align}
			\pi_t(x,y) \leq \frac{p_{\sigma^2}(y)}{p_{\sigma^2}(x)} e^{\frac{M \sqrt{d}}{2} t} g_t(x,y)
		\end{align}
		where $g_t(x,y)$ is the standard Gaussian heat kernel.
	\end{lemma}
	\begin{proof}
		We adapt a proof in \cite{Downes} to the case of higher dimensions.  By the same argument as in the proof of \Cref{prop6}, we may apply Girsanov's theorem.  Let $\mathbb{Q}_x$ be a measure under which $W(t) = \widetilde{B}_t$ is a $\mathbb{Q}_x$ Brownian motion started at $x$.  Let $\mathbb{P}_x$ be the original measure pertaining to the Brownian motion $B_t$ that drives $W(t)$.  Then by Girsanov's theorem, we have
		\begin{align}
			\left(\frac{d \mathbb{P}_x}{d \mathbb{Q}_x}\right)_t &= \widehat{\mathbb{E}}_x\left[\exp\left(\int_0^t \nabla \log p_{\sigma^2} (W(s)) d W(s) - \frac 12 \int_0^t \left|\left|\nabla \log p_{\sigma^2}(W(s))\right|\right|^2 d s \right)\right] \\
			&= \widehat{\mathbb{E}}_x\left[\exp\left(\int_0^t \nabla \log p_{\sigma^2}(W(s)) d \widetilde{B}_s - \frac 12 \int_0^t \left|\left|\nabla \log p_{\sigma^2}(W(s))\right|\right|^2 d s \right)\right]
		\end{align}
		where $\widehat{\mathbb{E}}_x$ denotes expectation with respect to $\mathbb{Q}_x$.  Now, by Rademacher's theorem we may apply Ito's lemma to $\log p_{\sigma^2}(\cdot)$, thus we get
		\begin{align}
			\log p_{\sigma^2}(W(t)) - \log p_{\sigma^2}(W(0)) = \int_0^t \nabla \log p_{\sigma^2}(W(s)) d W(s) + \frac 12 \int_0^t \Delta \log p_{\sigma^2}(W(s)) d s
		\end{align}
		Rearranging, we get that
		\begin{align}
			\int_0^t \nabla \log p_{\sigma^2}(W(S)) d \widetilde{B}_s = \log \left(\frac{p_{\sigma^2}(W(t))}{p_{\sigma^2}(x)}\right) - \frac 12 \int_0^t \Delta \log p_{\sigma^2}(W(s)) d s
		\end{align}
		By the fact that $\nabla \log p_{\sigma^2}$ is $\frac M2$-Lipschitz, we have that $| \Delta \log p_{\sigma^2}(y)| \leq \sqrt{d} M$ for all $y \in \mathbb{R}^d$.  Because $||\nabla \log p_{\sigma^2}(\cdot)||^2 \geq 0$, we have by the above work and the fact that the transition density of Brownian motion is $g_t(x,y)$, that
		\begin{align}
			\pi_t(x,y) \leq \frac{p_{\sigma^2}(y)}{p_{\sigma^2}(x)} e^{\frac{M \sqrt{d}}{2} t} g_t(x, y)
		\end{align}
		as desired.
	\end{proof}
	In order to bound the right hand side in \Cref{localtime}, we need to introduce the notion of assymmetric decreasing rearrangements.  A full exposition on the topic can be found in \cite{bennett1988}.  We have the following definition
	\begin{definition}
		Let $\mu$ be a probability measure on $\mathbb{R}^d$ and let $f$ be a nonnegative, measurable function $f: \mathbb{R}^d \to \mathbb{R}_{\geq 0}$.  We define for all $s \geq 0$,
		\begin{equation}
			\mu^f(s) = \mu \left(\left\{f(x) > s \right\} \right)
		\end{equation}
		We define the decreasing rearrangement of $f$ to be
		\begin{equation}
			f^*(t) = \inf\left\{s > 0 | \mu^f(s) \leq t \right\}
		\end{equation}
	\end{definition}
	Note that the specific case when $\mu$ is the Lebesgue measure is known as a symmetric decreasing rearrangement and is well known to geometric analysts.  In this special case, there is a well known inequality, the Hardy-Littlewood inequality that governs integrals of products of rearrangements.  A generalization to the arbitrary $\mu$ case, whose proof can be found in \cite{bennett1988}, is
	\begin{theorem}\label{HardyLittlewood}\cite[Theorem 2.2]{bennett1988}
		Let $f, g : \mathbb{R}^d \to \mathbb{R}_{\geq 0}$ be measurable functions.  Then
		\begin{equation}
			\int_{\mathbb{R}^d} f(x) g(x) d \mu(x) \leq \int_0^\infty f^*(t) g^*(t) d t
		\end{equation}
	\end{theorem}
	With this result in hand, we are able to prove the following proposition:
	\begin{proposition}\label{rearrangements}
		Considering the setting of \Cref{localtime}, let $U \subset \mathbb{R}^d$ be a measurable set.  If $d \geq 3$, then
		\begin{equation}
			\mathbb{E}_x\left[\int_0^t \mathbf{1}_U(W(s)) d s \right] \leq \frac{3 e^{\frac 1e}}{2 \pi} ||p_{\sigma^2}||_\infty^{1 - \frac 2d} \frac{e^{\frac{M \sqrt{d}}{2} t}}{p_{\sigma^2}(x)} p_{\sigma^2}(U)^{\frac 2d}
		\end{equation}
	\end{proposition}
	\begin{proof}
		By \Cref{localtime,parabolicpde}, we have
		\begin{align}
			\mathbb{E}_x\left[\int_0^t \mathbf{1}_U(W(s)) d s \right] &= \int_U \int_0^t \pi_s(x,y) d s d y \leq \int_U \int_0^t \frac{p_{\sigma^2}(y)}{p_{\sigma^2}(x)} e^{\frac{M \sqrt{d}}{2} s} g_s(x,y) d s d y \\
			&\leq \frac{e^{\frac{M \sqrt{d}}{2} t}}{p_{\sigma^2}(x)} \int_U \int_0^t g_s(x,y) p_{\sigma^2}(y) d s d y
		\end{align}
		Note now, that for $d \geq 3$,
		\begin{equation}
			\int_0^t g_s(x,y) d s \leq \int_0^\infty g_s(x,y) d s = \frac{\Gamma\left(\frac d2 - 1\right)}{2 \pi^{\frac d2}} ||x - y||^{2-d}
		\end{equation}
		by, for instance, \cite[Theorem 3.33]{Peres}.  Let $f(y) = \mathbf{1}_U(y)$ and let $g(y) = ||x - y||^{2 - d}$.  Then we note that
		\begin{align}
			p_{\sigma^2}^f(s) = \begin{cases}
				p_{\sigma^2}(U) & s < 1 \\
				0 & \text{otherwise}
			\end{cases} && p_{\sigma^2}^g(s) = p_{\sigma^2}\left(B_d\left(x, s^{- \frac 1{d-2}}\right)\right)
		\end{align}
		where $B_d(x, r)$ denotes the $d$-dimensional ball centred at $x$ with radius $r$.  Thus we have that
		\begin{align}
			f^*(t) = \mathbf{1}_{\{s < p_{\sigma^2}(U) \}}(t) && g^*(t) = \inf\left\{s | p_{\sigma^2}\left(B_d(x, s^{- \frac 1{d-2}})\right) \leq t \right\}
		\end{align}
		Let $\omega_d = \frac{\pi^{\frac d2}}{\Gamma\left(\frac d2 + 1\right)}$ be the Lebesgue volume of the unit ball.  Then we note that if $s = \left(\frac{t}{\omega_d ||p_{\sigma^2}||_\infty}\right)^{\frac 2d - 1}$, then we have
		\begin{equation}
			p_{\sigma^2}\left(B_d(x, s^{- \frac 1{d-2}})\right) \leq ||p_{\sigma^2}||_\infty \omega_d s^{- \frac{d}{d - 2}} = t
		\end{equation}
		Thus
		\begin{equation}
			g^*(t) \leq \left(\frac{t}{\omega_d ||p_{\sigma^2}||_\infty}\right)^{\frac 2d - 1}
		\end{equation}
		Now, applying \Cref{HardyLittlewood},  with $d \mu(y) = p_{\sigma^2}(y) d y$, we have
		\begin{align}
			\int_U \int_0^t g_s(x,y) d s p_{\sigma^2}(y) d y &\leq \int_U \frac{\Gamma\left(\frac d2 - 1\right)}{2 \pi^{\frac d2}} |x - y|^{2 - d} p_{\sigma^2}(y) d y \\
			&\leq \frac{\Gamma\left(\frac d2 - 1\right)}{2 \pi^{\frac d2}} \int_0^\infty \mathbf{1}_{\{s < p_{\sigma^2}(U)\}}(t) \left(\frac{t}{\omega_d ||p_{\sigma^2}||_\infty}\right)^{\frac 2d - 1} d t \\
			&= \frac{\Gamma\left(\frac d2 - 1\right)}{2 \pi^{\frac d2}}\int_0^{p_{\sigma^2}(U)} \left(\frac{t}{\omega_d ||p_{\sigma^2}||_\infty}\right)^{\frac 2d - 1} d t \\
			&= \frac{\Gamma\left(\frac d2 - 1\right)}{2 \pi^{\frac d2}} \omega_d^{1 - \frac 2d} ||p_{\sigma^2}||_\infty^{1 - \frac 2d} \frac d2 p_{\sigma^2}(U)^{\frac 2d}
		\end{align}
		Now, note that
		\begin{equation}
			\frac{\Gamma\left(\frac d2 - 1\right)}{2 \pi^{\frac d2}} \omega_d = \frac{2}{d(d - 2)}
		\end{equation}
		using the fact that $\Gamma(x+1) = x\Gamma(x)$.  Moreover, we have
		\begin{equation}
			\omega_d^{- \frac 2d} = \left(\frac{\Gamma\left(\frac d2 + 1\right)}{\pi^{\frac d2}}\right)^{\frac 2d} = \frac{\left(\frac d2\right)^{\frac 2d} \Gamma\left(\frac d2\right)^{\frac 2d}}{\pi} \leq \frac{e^{\frac 1e} \frac d2 }{\pi}
		\end{equation}
		by the fact that $x^{\frac 1x} \leq e^{\frac 1e}$ and the fact that $\Gamma(x) \leq x^x$.  Thus,
		\begin{equation}
			\int_U \int_0^t g_s(x,y) d s d y \leq ||p_{\sigma^2}||_\infty^{1 - \frac 2d} \frac{e^{\frac 1e}}{\pi} \frac d2 \frac{2}{d(d-2)} \frac d2 p_{\sigma^2}(U)^{\frac 2d} = \leq \frac{3 e^{\frac 1e}}{2 \pi} ||p_{\sigma^2}||_\infty^{1 - \frac 2d} p_{\sigma^2}(U)^{\frac 2d}
		\end{equation}
		Plugging this into the first set of inequalities above concludes the proof.
	\end{proof}
	With this in mind, we are able to prove the key lemma:
	\begin{lemma}\label{girsanovlemma}
		Let $\phi: \mathbb{R}^d \to \mathbb{R}_{\geq 0}$ be measurable and $W(t)$ as above and suppose that $\mathbb{E}[\phi(X)] \leq \varepsilon^2$, where $X$ is distributed according to $p_{\sigma^2}$.  Let $x \in \mathbb{R}^d$.  Then
		\begin{align}
			\mathbb{E}_x\left[\int_0^t \phi(W(s)) d s\right] \leq \varepsilon t + \sqrt{\left(\int_0^t \mathbb{E}_x[\phi(W(s))^2] d s\right)\frac{3 e^{\frac 1e}}{2 \pi} ||p_{\sigma^2}||_\infty^{1 - \frac 2d} \frac{e^{\frac{M \sqrt{d}}{2} t}}{p_{\sigma^2}(x)} \varepsilon^{\frac 2d}}
		\end{align}
		where $\mathbb{E}_x$ denotes expectation with respect to $W(t)$ started at $W(0) = x$.
	\end{lemma}	
	\begin{proof}
		Let $U = \left\{y \in \mathbb{R}^d \big| \phi(y) > \varepsilon \right\}$.  Then we have
		\begin{align}
			\int_0^t \phi(W(s))  d s = \int_0^t \phi(W(s)) \mathbf{1}_{U^c} d s + \int_0^t \phi(W(s)) \mathbf{1}_{U} d s \leq \varepsilon t + \int_0^t \phi(W(s)) \mathbf{1}_{U} d s
		\end{align}
		By Cauchy-Schwarz,
		\begin{align}
			\mathbb{E}_x\left[\int_0^t \phi(W(s)) \mathbf{1}_{U} d s \right] \leq \sqrt{\int_0^t \mathbb{E}_x\left[\phi(W(s))^2\right] d s \int_0^t \mathbb{E}_x\left[\mathbf{1}_U\right] d s}
		\end{align}
		By assumption, we have
		\begin{align}
		\varepsilon^2 \geq \mathbb{E}[\phi(X)] \geq \mathbb{E}\left[ \phi(X) \mathbf{1}_U\right] \geq \varepsilon \int_U p_{\sigma^2}(y) d y
		\end{align}
		Thus we have that $p_{\sigma^2}(U) \leq \varepsilon$.  Applying \Cref{rearrangements} to bound the second factor under the square root concludes the proof.
	\end{proof}

	Finally, we are able to prove \Cref{prop3}.
	\begin{proof}(Proof of \Cref{prop3})
		We apply \Cref{girsanovlemma} to $\phi(x) = ||\nabla \log p(x) - \widehat{f}(x)||^2$.  By the Lipschitz condition and \Cref{growth}, we note that
		\begin{align}
			\phi(x)^2 \leq 16 M^4 ||x||^4 + 16B^4
		\end{align}
		and by \Cref{fourthmoment}, we have then
		\begin{align}
			\int_0^t\mathbb{E}_x\left[\phi(W(s))^2\right] d s \leq \left(16 M^4 \left(\mathbb{E}\left[||W(0)||^4\right] + \frac{(b + d + 2)\left(\mathbb{E}\left[||W(0)||^2\right] + \frac{b + d}{m}\right)}{m}\right) + 16 B^4\right) t
		\end{align}
		Thus there exists a constant $C$ depending on the initialization and the smoothness parameters $m, b, M, B$ such that
		\begin{equation}
			\int_0^t\mathbb{E}_x\left[\phi(W(s))^2\right] d s \leq C d^2 t
		\end{equation}
		Now, note that since $W(0)$ is concentrated with high probability in a ball of radius $R$, we have by \Cref{growth} that $p_{\sigma^2}(W(0)) \geq Ce^{- MR^2 - B R}$.  Thus by \Cref{girsanovlemma}, we have that
		\begin{align}
			\mathbb{E}\left[\int_0^t \phi(W(s)) d s\right] \leq \varepsilon t + \sqrt{C d^2 t e^{M R^2 + BR} \frac{3 e^{\frac 1e}}{2 \pi} ||p_{\sigma^2}||_\infty^{1 - \frac 2d} e^{\frac{M \sqrt{d}}{2} t} \varepsilon^{\frac 2d}} \leq \varepsilon t + C ||p||_\infty^{\frac 12 - \frac 1d} e^{\frac{M \sqrt{d}}{4} t} \sqrt{t} \varepsilon^{\frac 1d}
		\end{align}
		as desired.
	\end{proof}

\section{Miscellanious Proofs}\label{app1}
    \begin{proof}(Proof of \Cref{prop1})
        Let $\epsilon \sim g_{\sigma^2}$ and $X \sim p$.  Then we have, letting $y = x + \epsilon$,
        \begin{align}
            L_{DAE}(r) &= \mathbb{E}_X \mathbb{E}_\epsilon\left[ ||r(x + \epsilon) - x||^2 \right] = \int\int ||r(y) - y + \epsilon||^2 p(y-\epsilon) g(\epsilon) d y d \epsilon \\
            &= \int\int ||r(y) - y||^2 p(y -\epsilon)g(\epsilon) d y d \epsilon +\int\int 2 \langle \epsilon, r(y)-y\rangle p(y-\epsilon) g(\epsilon) d y d \epsilon \\
            &+ \int\int ||\epsilon||^2 p(y-\epsilon)g(\epsilon) d y d \epsilon
        \end{align}
        The last term above does not depend on $r$ and so we may ignore it.  We focus now on the second term.  Let $\xi \sim g_1$ be a standard Gaussian and let $s'(x) = r(x) -x$.  By \Cref{steinlemma}, we have that
        \begin{eqnarray}
            \int \langle \epsilon, s'(x+\epsilon) \rangle  g(\epsilon) d \epsilon &=\sigma \int \langle \xi,s'(x+\sigma \xi )\rangle g_{1}(\xi)d\xi \\
            &= \sigma^2 \frac{1}{\sigma}\int \langle \xi,s'(x+\sigma \xi )\rangle g_{1}(\xi)d\xi\\
            &= \sigma^2  \divv(\int s'(x+\epsilon) g(\epsilon) d \epsilon)
        \end{eqnarray}
        where we used Gaussian Stein identity above. Now, note that as we know that $p_{\sigma^2}$ is a density, it must tend to zero as $||x|| \to \infty$.  Thus we may apply the divergence theorem to get
        \begin{align}
            \int\int 2 \langle \epsilon, r(y)-y\rangle p(y-\epsilon) g(\epsilon) d y d \epsilon &=  2 \sigma^2 \int p(x) \divv \left( \mathbb{E}_\epsilon [s'(x+\epsilon)]\right) d x \\
            &=  2\sigma^2 \int \divv(s'(y)) p_{\sigma^2}(y) d y \\
            &=- 2 \sigma^2 \int \left\langle s'(x), \frac{\nabla p_{\sigma^2}(x)}{p_{\sigma^2}(x)} \right\rangle p_{\sigma^2}(x) d x
        \end{align}
        Thus we have that
        \begin{align}
            L_{DAE}(r) &= \mathbb{E}_{X \sim p_{\sigma^2}}\left[ ||r(X) - X||^2 - 2 \sigma^2 \langle s'(X), \nabla \log p_{\sigma^2}(X)\rangle \right] + C(p,\sigma^2) \\
            &= \mathbb{E}_{X \sim p_{\sigma^2}}\left[\left|\left| s'(X) - \sigma^2 \nabla \log p_{\sigma^2}(X)\right|\right|^2 \right] + C(p, \sigma^2) - \sigma^4 \mathbb{E}_{X \sim p_{\sigma^2}} [||\nabla \log p_{\sigma^2}||^2] \\
            &= \mathbb{E}_{X \sim p_{\sigma^2}}\left[\left|\left| s'(X) - \sigma^2 \nabla \log p_{\sigma^2}(X)\right|\right|^2 \right] + C'(p, \sigma^2)
        \end{align}
        where $C(p, \sigma^2)\,,\,C'(p, \sigma^2)$ do not depend on $r$.  Dividing by $\sigma^2$ and setting $s(x) = \frac{s'(x)}{\sigma^2}$ shows that
        \begin{equation}
            L_{DSM}(s) = \mathbb{E}_{p_{\sigma^2}}[||s(x) - \nabla \log p_{\sigma^2}||^2 ] = \frac 1{\sigma^2} L_{DAE}(r) + C(p, \sigma^2)
        \end{equation}
        Thus, the two losses are equivalent to minimize with respect to $r$ or $s$.
    \end{proof}
	\begin{proof}(Alternate proof of \Cref{firstcor})
		For $y \in \mathbb{R}^d$, the loss of the DAE is given by
		\begin{equation}
			\int_{\mathbb{R}^d} \mathbb{E}_\epsilon \left[p(y) || r(y + \epsilon) - y||^2 \right] d y = \int_{\mathbb{R}^d} \mathbb{E}_\epsilon \left[p(x - \epsilon) ||r(x) - x + \epsilon ||^2 \right] d x
		\end{equation}
		where we substituted $x = y + \epsilon$.  Now by the calculus of variations, it suffices to minimize the integrand with respect to $r(x)$ for each $x \in \mathbb{R}^d$.  Taking the derivative and setting it equal to zero gives
		\begin{equation}
			r_{\sigma^2}(x) = \frac{\mathbb{E}_\epsilon \left[p(x - \epsilon) (x - \epsilon)\right]}{\mathbb{E}_\epsilon \left[p(x - \epsilon)\right]} 
		\end{equation}
		the result given by \cite[Theorem 1]{Alain}.  By linearity, then we have
		\begin{equation}
			r_{\sigma^2}(x) = \frac{\mathbb{E}_\epsilon \left[x p(x-\epsilon)\right]}{\mathbb{E}_\epsilon \left[p(x - \epsilon)\right]} - \frac{\mathbb{E}_\epsilon\left[\epsilon p(x - \epsilon)\right]}{\mathbb{E}_\epsilon \left[p(x-\epsilon)\right]} = x - \frac{\mathbb{E}_\epsilon\left[\epsilon p(x - \epsilon)\right]}{\mathbb{E}_\epsilon \left[p(x-\epsilon)\right]}
		\end{equation}
		But we have
		\begin{equation}
			\mathbb{E}_\epsilon \left[\epsilon p(x - \epsilon)\right] = \int \epsilon p(x - \epsilon) g_{\sigma^2}(\epsilon) d \epsilon = \sigma^2 \int \nabla g_{\sigma^2}(\epsilon) p(x - \epsilon) d \epsilon = - \sigma^2 \int g_{\sigma^2}(\epsilon) \nabla p(x - \epsilon) d \epsilon
		\end{equation}
		by \Cref{steinlemma}.  But then we have
		\begin{equation}
			\frac{\mathbb{E}_\epsilon\left[\epsilon p(x - \epsilon)\right]}{\mathbb{E}_\epsilon \left[p(x-\epsilon)\right]} = - \sigma^2\frac{\mathbb{E}_\epsilon\left[- \nabla_x p(x - \epsilon) \right]}{\mathbb{E}_\epsilon \left[p(x - \epsilon)\right]} = \sigma^2 \frac{\nabla p \ast g_{\sigma^2}(x)}{p \ast g_{\sigma^2}(x)} = \sigma^2 \nabla \log p_{\sigma^2}
		\end{equation}
		Putting this together yields the result.
	\end{proof}
	\begin{proof}(Proof of \Cref{cor2})
	Let $\sigma^2 \leq \sigma_{max}^2$. Let $\eta_{\sigma^2}(X)= \nabla \log p_{\sigma^2}(X)- \nabla \log p(X) $.  By \Cref{cor1}, $\eta_{\sigma^2}$ is Lipschitz with constant $\sigma^2 \frac M2$.  Thus, we have : 
	\begin{align}
	\langle - \nabla \log p_{\sigma^2}(x), x\rangle = \langle - \nabla \log p(x), x \rangle -  \langle \eta_{\sigma^2}(X), x \rangle
	\end{align}
	By \Cref{growth} we know that there is some constant, which, by raising $B$ if necessary, we may take to be equal to $B$, such that
	\begin{equation}
	||\eta_{\sigma^2}(x)|| \leq \sigma^2 M ||x|| + B
	\end{equation}
	By Cauchy-Schwarz,
	\begin{equation}
	\left| \langle \eta_{\sigma^2}(X), x \rangle \right| \leq ||\eta_{\sigma^2}(x)|| \cdot  ||x|| \leq  \sigma^2 M ||x||^2 + B ||x||
	\end{equation}
	and thus
	\begin{equation}
	\langle \eta_{\sigma^2}(X), x \rangle  \geq - \sigma^2 M ||x||^2 - B ||x||
	\end{equation}
	By the dissipativity assumption, we have
	\begin{equation}
	\langle - \nabla \log p(x), x \rangle \geq m ||x||^2 - b
	\end{equation}
	Thus we have
	\begin{align}
	\langle - \nabla \log p_{\sigma^2}(x), x \rangle &\geq (m - \sigma^2 M ) ||x||^2 - b -  B ||x|| \\
	&\geq \frac{m - \sigma^2 M }{2} ||x||^2 - b + \frac{m - \sigma^2 M}{2} ||x||^2 -  B ||x|| \\
	&\geq \frac{m -  \sigma^2 M}{2} ||x||^2 -b -  \frac{B^2}{2(m -  \sigma^2 M)}
	\end{align}
	where the last inequality follows by the fact that
	\begin{align}
	\frac{m - \sigma^2 M }{2} ||x||^2 -  B ||x||  &= \frac{m -  \sigma^2 M}{2} \left( ||x|| - \frac{ B}{m -  \sigma^2 M}\right)^2 -\frac{ B^2}{2 (m -  \sigma^2 M)} \\
	&\geq -\frac{ B^2}{2 (m -  \sigma^2 M)}
	\end{align}
\end{proof}

    \begin{proof}(Proof of \Cref{prop2})
		Consider the coupling where the Brownian motion driving $W(t)$ also generates the Gaussians in $W_k$.  Let $\widehat{W}(s)$ be the continuous time process such that $\widehat{W}\left(s \right) = W_{\left\lfloor \frac{s}{\eta} \right\rfloor \eta }$ and $\widehat{W}(0) = W(0)$.  Recall $\tau = k \eta$.  Then we can compute
		\begin{align*}
			\mathbb{E}\left[\left|\left|W(\tau) - W_{k} \right|\right|^2 \right] &= \mathbb{E} \left[\left|\left|W(\tau) - \widehat{W}(\tau) \right|\right|^2 \right] = \mathbb{E}\left[\left|\left|\int_0^{\tau} \nabla \log p(W(s)) - \nabla \log p(\widehat{W}(s)) d s\right|\right|^2 \right] \\
			&\leq \int_0^{\tau} \mathbb{E}\left[\left|\left|\nabla \log p(W(s)) - \nabla \log p(\widehat{W}(s)) \right|\right|^2\right] d s \\
			&\leq M^2 \sum_{j = 0}^{k-1} \int_{j\eta}^{(j+1) \eta} \mathbb{E}\left[\left|\left|W(s) - W_j \right|\right|^2\right] d s
		\end{align*}
		Now note that for $j \eta \leq s < (j+1) \eta$, we have
		\begin{align*}
			\mathbb{E}\left[\left|\left|W(s) - W_j \right|\right|^2\right] \leq  2 \mathbb{E}\left[\left|\left|W(s) - W(j \eta) \right|\right|^2\right] + 2 \mathbb{E}\left[\left|\left|W(j\eta) - W_j \right|\right|^2\right]
		\end{align*}
		For the first term, we have
		\begin{align}
			\mathbb{E}\left[\left|\left|W(s) - W_j \right|\right|^2\right] &= \mathbb{E}\left[\left|\left|\int_{j\eta}^s \nabla \log p(W(u)) d u + \sqrt{2} \int_{j\eta}^s d B_u \right|\right|^2\right] \\
			&\leq 2 \eta M^2 \mathbb{E}\left[||W(s)||^2 \right] + 4d\eta \leq 2 \eta M^2 \mathbb{E}\left[||W(\tau)||^2 \right] + 4d \eta 
		\end{align}
		where we recall that $\tau = k \eta$.  Plugging this last bound into our above inequality, we get
		\begin{align*}
			\mathbb{E}\left[\left|\left|W(\tau) - W_{k} \right|\right|^2 \right] \leq M^2 \tau \left(2 \eta M^2 \mathbb{E}\left[||W(\tau)||^2 \right] + 4d \eta \right) + M^2 \eta \sum_{j = 0}^{k-1} \mathbb{E}\left[\left|\left|W(j \eta) - W_{j\eta} \right|\right|^2 \right]
		\end{align*}
		Applying the discrete Gronwall lemma, we have
		\begin{align}
			\mathbb{E}\left[\left|\left|W(\tau) - W_{k} \right|\right|^2 \right] \leq M^2 \tau \left(2 \eta M^2 \mathbb{E}\left[||W(\tau)||^2 \right] + 4d \eta \right) e^{M^2 \tau}
		\end{align}
		To conclude, we apply \Cref{fourthmoment} to bound $\mathbb{E}\left[||W(\tau)||^2\right]$ and the result follows from the fact that the Wasserstein distance is bounded by the above explicit coupling of the two laws $\mu_k$ and $\nu_{\tau}$ to get
		\begin{align}
			\mathcal{W}_2\left(\mu_k, \nu_{\tau}\right)^2 \leq M^2 \tau \eta \left(4 d + M^2 \left(\mathbb{E}\left[||W(0)||^2\right] + \frac{b + d}{m} \right)\right) e^{M^2 \tau} \leq C d \eta \tau e^{M^2 \tau}
		\end{align}
		with $C$ depending on $M, b, m$, and $\mathbb{E}\left[||W(0)||^2\right]$ as desired.
	\end{proof}

\section{Auxiliary Lemmata}
    \begin{lemma}\label{growth}
        If $\nabla \log p$ is $\frac M2$-Lipschitz, then there is some constant $B$ such that
        \begin{equation}
            ||\nabla \log p(x)||^p \leq M^p ||x||^p + B^p
        \end{equation}
        for all $p \geq 1.$
    \end{lemma}
    \begin{proof}
        By the definition of Lipschitz, we have that
        \begin{equation*}
            ||\nabla \log p(x)|| \leq ||\nabla \log p(x) - \nabla \log p(0)|| + ||\nabla \log p(0)|| \leq  \frac M2 ||x|| + ||\nabla \log p(0)||
        \end{equation*}
        Applying Minkowski's inequality concludes the proof.
    \end{proof}
    \begin{lemma}\label{moments}
	    For all $p > 1$ and all $t > 0$, we have that if $\nabla \log p$ is $\frac M2$ Lipschitz, then
	    \begin{equation}
		    \mathbb{E}\left[||W(t)||^p\right] \leq \left(\left(2 d t\right)^{\frac p2} + B^p t \right) e^{M^p t}
		\end{equation}
		where $B$ is as appears in \Cref{growth}.
	\end{lemma}
	\begin{proof}
		By the triangle inequality and \Cref{growth}, we have
		\begin{align}
			||W(t)||^p &= \left|\left|\sqrt{2} B_t + \int_0^t \nabla \log p(W(s)) d s \right|\right|^p \leq 2^{\frac p2} ||B_t||^p + \int_0^t ||\nabla \log p(W(s))||^p d s \\
			&\leq 2^{\frac p2} ||B_t||^p + \int_0^t \left(M^p ||W(s)||^p + B^p \right) d s
		\end{align}
		Taking expected values and applying Fubini, we get
		\begin{align}
			\mathbb{E}\left[||W(t)||^p\right] &\leq 2^{\frac p2} \mathbb{E}\left[||B_t||^p\right] + \int_0^t \left(M^p \mathbb{E}\left[||W(s)||^p\right] + B^p \right) d s \\
			&\leq \left(2dt\right)^{\frac p2} + B^p t + M^p\int_0^t \mathbb{E}\left[||W(s)||^p\right] d s
		\end{align}
		Applying Gronwall's inequality finishes the proof.
	\end{proof}
	\begin{remark}
		The following lemma will demonstrate that the bound in \Cref{moments} is not tight as an inductive argument applied to the result below would show that each moment of $W(t)$ is actually bounded uniformly in time.  The utility of the above lemma is that we have finiteness of all moments without the extra difficulties of iterating the proof below.
	\end{remark}
	\begin{lemma}\label{fourthmoment}
		There exists a constant $\kappa$ depending only on the initialization of $W_{\sigma^2}(0)$ such that for all $t$,
		\begin{align}
			\mathbb{E}\left[||W_{\sigma^2}(t)||^2\right] &\leq \mathbb{E}\left[||W(0)||^2\right] e^{-2m t} + \frac{b + d}{m} \\
			\mathbb{E}\left[||W_{\sigma^2}(t)||^4\right] &\leq \mathbb{E}\left[||W(0)||^4\right] + \frac{(b + d + 2)\left(\mathbb{E}\left[||W(0)||^2\right] + \frac{b + d}{m}\right)}{m}
		\end{align}
	\end{lemma}
	\begin{proof}
		We adapt the proof of \cite[Lemma 3.2]{RakhlinRaginsky}.  Without loss of generality, we take $\sigma^2 = 0$ as the proof relies only on the Lipschitz and dissipative constants.
		
		Let $Y(t) = ||W(t)||^4$.  By Ito's lemma, then, we have
		\begin{equation}
			d Y(t) = 4 ||W(t)||^2 \langle W(t), \nabla \log p(W(t)) \rangle d t + 4 (d  +2) ||W(t)||^2 + 4 ||W(t)||^2 W(t) \sqrt{2} d B_t
		\end{equation}
		Thus
		\begin{align}
			d\left(e^{4mt} Y(t)\right) = &4m e^{4m t} ||W(t)||^4 d t + 4 e^{4m t} ||W(t)||^2 \langle W(t), \nabla \log p(W(t)) \rangle d t \\
			&+ 4  e^{4m t}(d  +2) ||W(t)||^2 + 4 ||W(t)||^2 W(t) \sqrt{2} e^{4m t} d B_t
		\end{align}
		Thus we have
		\begin{align}
			Y(t) = &e^{-4mt} Y(0) + \int_0^t e^{4m(s-t)}4 ||W(s)||^2 \left(\langle W(s), \nabla \log p(W(s)) \rangle + m ||W(s)||^2\right) d s\\
			&+ \int_0^t e^{4m(s-t)} 4 (d + 2) ||W(s)||^2 d s + \int_0^t e^{4m(s-t)} \sqrt{2} 4 ||W(s)||^2 W(s) d B_s
		\end{align}
		By \Cref{moments}, the last term is an actual martingale and so has expectation zero.  By the dissipativity assumption, we have
		\begin{align}
			\langle W(s), \nabla \log p(W(s)) \rangle + m ||W(s)||^2 &= - \langle - \nabla \log p(W(s)), W(s) \rangle + m ||W(s)||^2 \\
			&\leq - m ||W(s)||^2 + b + m||W(s)||^2 = b
		\end{align}
		Thus we have
		\begin{align}
			\mathbb{E}[Y(t)] \leq e^{-4mt} \mathbb{E}[Y(0)] + \int_0^t e^{4m(s-t)} 4(b + d + 2) \mathbb{E}\left[||W(s)||^2\right] d s 
		\end{align}
		Now, we need to bound $\mathbb{E}\left[||W(s)||^2\right]$ independently of $s$.  We repeat the same trick from \cite{RakhlinRaginsky}.  We define $Y'(t) = ||W(t)||^2$.  Then
		\begin{align}
			d Y'(t) = 2 \langle W(t), \nabla \log p(W(t)) \rangle d t + 2 d d t + \sqrt{2} W(t) d B_t
		\end{align}
		Thus
		\begin{align}
			d\left(e^{2mt} Y'(t)\right) = 2 m e^{2mt} ||W(t)||^2 d t + 2 \langle W(t), \nabla \log p(W(t)) \rangle d t + 2 d d t + \sqrt{2} W(t) d B_t
		\end{align}
		Again, we have
		\begin{align}
			Y'(t) = &e^{-2mt} Y'(0) + \int_0^t e^{2m(s-t)}2 \left(\langle W(s), \nabla \log p(W(s)\rangle + m||W(s)||^2\right) d s \\
			&+ \int_0^t e^{2m(s -t)} 2d d t + \int_0^t e^{2m (s -t)} \sqrt{2} W(t) d B_t
		\end{align}
		Again, we note that the first integrand is bounded by $e^{2m(s-t)} 2 b$ by dissipativity and the last integral drops out in expectation.  Thus
		\begin{equation}
			\mathbb{E}\left[||W(t)||^2\right] \leq e^{-2mt} \mathbb{E}\left[||W(0)||^2\right] + \int_0^t e^{2m(s -t)} 2(b + d) d t = e^{-2mt} \mathbb{E}\left[||W(0)||^2\right] + \frac{b + d}{m}\left(1 - e^{-2mt}\right)
		\end{equation}
		which establishes the first inequality.  For the second inequality, we see that
		\begin{align}
			\mathbb{E}\left[||W(t)||^4\right] &\leq e^{-4mt} \mathbb{E}\left[||W(0)||^4\right] + \int_0^t e^{4m(s-t)} 4(b + d + 2) \left(\mathbb{E}\left[||W(0)||^2\right] + \frac{b + d}{m}\right) d s \\
			&= e^{-4mt} \mathbb{E}\left[||W(0)||^4\right] + \frac{(b + d + 2)\left(\mathbb{E}\left[||W(0)||^2\right] + \frac{b + d}{m}\right)}{m} \left(1 - e^{-4mt}\right) \\
			&\leq \mathbb{E}\left[||W(0)||^4\right] + \frac{(b + d + 2)\left(\mathbb{E}\left[||W(0)||^2\right] + \frac{b + d}{m}\right)}{m}
		\end{align}
		as desired.
	\end{proof}
	\begin{lemma}\label{exponential}
		Suppose that $- \nabla \log p_{\sigma^2}$ is $\frac M2$-Lipschitz and $(m,b)$-dissipative.  Suppose that $W(t)$ is a solution to \Cref{eq:langevin} initialized at $W(0)$ such that there is some $0 < \alpha \leq m$ such that $\log \mathbb{E}\left[e^{\alpha ||W(0)||^2}\right] = k_{\alpha} < \infty$.  Then for all $t > 0$,
		\begin{equation}
			\log \mathbb{E}\left[e^{\alpha ||W(t)||^2}\right] \leq k_{\alpha} + 2 \alpha(b + d) t.
		\end{equation}
            If $\widehat{f}$ is also $\frac M2$-Lipschitz and $(m, b)$-dissipative, then the same result holds for $\widehat{W}(t)$, i.e.,
            \begin{equation}
                \log \mathbb{E}\left[e^{\alpha ||W(t)||^2}\right] \leq k_{\alpha} + 2 \alpha(b + d)t.
            \end{equation}
	\end{lemma}
	\begin{remark}
		Note that this is essentially \cite[Lemma 3.3]{RakhlinRaginsky} in our setting and is proved in the same way.  The proof is included for the sake of completeness.
	\end{remark}
	\begin{proof}
		As above, without loss of generality, we take $\sigma^2 = 0$.  Let $Y(t) = e^{\alpha ||W(t)||^2}$.  By the Ito calculus, we have
		\begin{align}
			d Y(t) &= 2 \alpha W(t) Y(t) d W(t) + 2d \alpha Y(t) d t + 2 \alpha^2 ||W(t)||^2 Y(t) d t \\
			&= \left(2 \langle W(t), \nabla \log p(W(t)) \rangle + 2d + 2\alpha ||W(t)||^2 \right) \alpha Y(t) d t + 2 \sqrt{2} \alpha Y(t) W(t) d B_t
		\end{align}
	Thus
	\begin{equation}
		Y(t) = Y(0) + \int_0^t \left(2 \langle W(s), \nabla \log p(W(s)) \rangle + 2d + 2\alpha ||W(s)||^2 \right) \alpha Y(s) d s + \int_0^t 2 \sqrt{2} \alpha Y(s) W(s) d B_s
	\end{equation}
	By \cite[Corollary 4.1]{Djellout}, we know $\int_0^t \mathbb{E}[||Y(s)W(s)||^2] ds < \infty$ and thus the last term above is a real martingale.  Taking expectations, we have by assumption that $\mathbb{E}\left[Y(0)\right] = e^{k_\alpha}$ and thus we have
	\begin{align}
		\mathbb{E}[Y(t)] = e^{k_\alpha} + \mathbb{E}\left[ \int_0^t \left(2 \langle W(s), \nabla \log p(W(s)) \rangle + 2d + 2\alpha ||W(s)||^2 \right) \alpha Y(s) d s\right]
	\end{align}
	Now, we note that by dissipativity, we have
	\begin{align}
		2 \langle W(s), \nabla \log p(W(s)) \rangle + 2d + 2\alpha ||W(s)||^2 \leq 2 b - 2 m ||W(s)||^2 + 2d + 2 \alpha ||W(s)||^2 \leq 2 b + 2 d
	\end{align}
	by the assumption that $\alpha \leq m$.  Thus we have
	\begin{align}
		\mathbb{E}[Y(t)] \leq e^{k_\alpha} + \int_0^t Y(s) 2 \alpha (b + d) d s
	\end{align}
	By Gronwall's inequality, the first result follows.  Applying the identical argument to $\widehat{W}(t)$ concludes the proof.
	\end{proof}

    \begin{lemma}\label{mingrowth}
        Let $-\nabla \log p$ be $(m,b)$-dissipative.  Then
        \begin{equation}
            \log p(x) \leq - \frac m4 ||x||^2 + \frac{2b^2}{m} + \log p(0)
        \end{equation}
    \end{lemma}
    \begin{proof}
        By the fundamental theorem of calculus,
        \begin{align}
            \log p(x) &= \log p(0) + \int_0^1 \frac{d}{dt}\left( \log p(tx)\right) d t = \log p(0) + \int_0^1 \langle \nabla \log p(tx), x \rangle d t \\
            &\leq \log p(0) - \int_0^1 m t^2|| x||^2 + b d t = \log p(0) + b - \frac m2 ||x||^2 \\
            &= - \frac m4 ||x||^2  + \log p(0) - \frac m4 ||x||^2 + b ||x||
        \end{align}
        by the dissapitivity assumption.  Maximizing the last two terms with respect to $||x||$ yields the result.
    \end{proof}
    \begin{lemma}\label{subgaussian}
        Let $-\nabla \log p$ be $(m,b)$-dissapitive.  Then there is a constant $C$ such that for all $||x||$,
        \begin{equation}
            p(x) \leq C e^{- \frac{m ||x||^2}{4}}
        \end{equation}
        In particular, $p$ is $\frac 2m$-sub-Gaussian.
    \end{lemma}
    \begin{proof}
        The first statement follows immediately from \Cref{mingrowth}.  The second follows from a Gaussian tail-bound.
    \end{proof}
    \begin{lemma}\label{steinlemma}(Gaussian Stein Identity, \cite{Stein1981})
        Let $\xi \sim N(0, I_d)$ and let $g: \mathbb{R}^d \to \mathbb{R}$ be an almost everywhere differentiable function with $\mathbb{E}_\xi [||\nabla g(\xi)||] < \infty$.  Then
        \begin{equation}
            \mathbb{E}_\xi \left[g(\xi) \xi \right] = \mathbb{E}_\xi \left[ \nabla g(\xi) \right]
        \end{equation}
    \end{lemma}
\end{document}